\documentclass{article}

\usepackage[utf8]{inputenc} 
\usepackage[T1]{fontenc}    
\usepackage{lmodern}

\usepackage{url}            
\usepackage{booktabs}       
\usepackage{amsfonts}       
\usepackage{nicefrac}       
\usepackage{microtype}      
\usepackage{amsmath,amssymb,amsbsy,amstext,amsthm,mathrsfs}
\usepackage{color}
\usepackage{algorithm}
\usepackage{subfigure}
\usepackage{latexsym,graphics,epsf,epsfig,psfrag,wrapfig}
\usepackage{booktabs}
\usepackage{cite}

\newcommand{\bz}{\boldsymbol{z}}
\newcommand{\bx}{\boldsymbol{x}}

\newcommand{\ba}{\boldsymbol{a}}
\newcommand{\bb}{\boldsymbol{b}}

\newcommand{\by}{\boldsymbol{y}}
\newcommand{\bw}{\boldsymbol{w}}
\newcommand{\bY}{\boldsymbol{Y}}

\newcommand{\bU}{\boldsymbol{U}}
\newcommand{\bzero}{\boldsymbol{0}}
\newcommand{\bone}{\boldsymbol{1}}

\newcommand{\bI}{\boldsymbol{I}}
\newcommand{\cE}{\mathcal{E}}

\newcommand{\bh}{\boldsymbol{h}}

\newcommand{\bA}{\boldsymbol{A}}
\newcommand{\bB}{\boldsymbol{B}}
\newcommand{\bF}{\boldsymbol{F}}
\newcommand{\bD}{\boldsymbol{D}}
\newcommand{\bM}{\boldsymbol{M}}
\newcommand{\bbP}{\mathbb{P}}
\newcommand{\bbR}{\mathbb{R}}
\newcommand{\bbN}{\mathbb{N}}
\newcommand{\bbC}{\mathbb{C}}

\newcommand{\bv}{\boldsymbol{v}}

\newcommand{\cS}{\mathcal{S}}

\newcommand{\cN}{\mathcal{N}}

\newcommand{\cO}{\mathcal{O}}
\newcommand{\cP}{\mathcal{P}}
\newcommand{\cI}{\mathcal{I}}

\newcommand{\cA}{\mathcal{A}}

\newcommand{\sgn}{\text{sgn}}
\newcommand{\phase}{\text{Ph}}

\newcommand{\erfc}{\text{erfc}}
\newcommand{\erf}{\text{erf}}
\newcommand{\dist}{\text{dist}}

\newcommand{\mE}{\mathrm{E}}

\newtheorem{lemma}{\textbf{Lemma}}
\newtheorem{theorem}{\textbf{Theorem}}
\newtheorem{problem}{\textbf{Problem}}
\newtheorem{proposition}{\textbf{Proposition}}

\newcommand{\nn}{\nonumber}

\title{Reshaped Wirtinger Flow and Incremental Algorithm  \\ for Solving Quadratic System of Equations}

\author{ Huishuai Zhang, Yi Zhou, Yingbin Liang, Yuejie Chi,\thanks{H. Zhang, Y. Zhou and Y. Liang are with Department of EECS, Syracuse University, Syracuse, NY 13244 USA (email: \{hzhan23,yzhou35, yliang06\}@syr.edu). Y. Chi is with Department of ECE, The Ohio State University, Columbus, OH 43210 USA (email: chi.97@osu.edu).} }

\begin{document}

\maketitle

\begin{abstract}

We study the phase retrieval problem, which solves quadratic system of equations, i.e., recovers a vector $\boldsymbol{x}\in \mathbb{R}^n$ from its magnitude measurements $y_i=|\langle \boldsymbol{a}_i, \boldsymbol{x}\rangle|, i=1,..., m$. We develop a gradient-like algorithm (referred to as RWF representing reshaped Wirtinger flow) by minimizing a nonconvex nonsmooth loss function.  In comparison with existing nonconvex  Wirtinger flow (WF) algorithm \cite{candes2015phase}, although the loss function becomes nonsmooth, it involves only the second power of variable and hence reduces the complexity. We show that for random Gaussian measurements, RWF enjoys geometric convergence to a global optimal point as long as the number $m$ of measurements is on the order of $n$, the dimension of the unknown $\boldsymbol{x}$. This improves the sample complexity of WF, and achieves  the same sample complexity as truncated Wirtinger flow (TWF) \cite{chen2015solving}, but without truncation in gradient loop. Furthermore, RWF costs less computationally than WF, and runs faster numerically than both WF and TWF. 
We further develop the incremental (stochastic) reshaped Wirtinger flow (IRWF) and show that IRWF converges linearly to the true signal. We further establish performance guarantee of an existing Kaczmarz method for the phase retrieval problem based on its connection to IRWF.  We also empirically demonstrate that IRWF outperforms existing ITWF algorithm (stochastic version of TWF) as well as other batch algorithms.
\end{abstract}

\section{Introduction}

Many problems in machine learning and signal processing can be reduced to solve a quadratic system of equations. For instance, in phase retrieval applications, i.e., X-ray crystallography and coherent diffraction imaging \cite{drenth2007X,miao1999extending,miao2008extending}, the structure of an object is to be recovered from the measured far field diffracted intensity when an object is illuminated by a source light. Mathematically, such a problem amounts to recover the signal from only measurements of magnitudes. Specifically, the problem is formulated below.
\begin{problem} \label{pb:mainproblem}
	Recover $\bx \in \bbR^n/\bbC^n$ from measurements $y_i$ given by
	\begin{flalign}
		y_i=\left|\langle \ba_i,\mathbf{x}\rangle\right|, \quad \text{for }\; i=1,\cdots,m, \label{eq:mainproblem}
	\end{flalign}
	where $\ba_i \in \bbR^n/\bbC^n$ are random design vectors (known). 
\end{problem}

Various algorithms have been proposed to solve this problem since 1970s. The error-reduction methods proposed in \cite{gerchberg1972practical, fienup1982phase} work well empirically but lack theoretical guarantees. More recently, convex relaxation of the problem has been formulated, for example, via PhaseLift \cite{chai2011array,candes2013phaselift, gross2015improved} and PhaseCut \cite{waldspurger2015phase}, and the correspondingly developed algorithms typically come with performance guarantee.  The reader can refer to the review paper \cite{shechtman2015phase} to learn more about applications and algorithms of the phase retrieval problem.

While with good theoretical guarantee, these convex methods often suffer from computational complexity particularly when the signal dimension is large. On the other hand, more efficient nonconvex approaches have been proposed and shown to recover the true signal as long as initialization is good enough. \cite{sujay2013phase} proposed \emph{AltMinPhase} algorithm, which alternatively updates the phase and the signal with each signal update solving a least-squares problem, and showed that AltMinPhase converges linearly and recovers the true signal with $\cO(n\log^3n)$ Gaussian measurements. More recently, \cite{candes2015phase} introduces \emph{Wirtinger flow} (WF) algorithm, which guarantees signal recovery via a simple gradient algorithm with only $\cO(n\log n)$ Gaussian measurements and attains $\epsilon-$accuracy within $\cO(mn^2\log 1/\epsilon)$ flops. More specifically, WF obtains good initialization by the spectral method, and then minimizes the following nonconvex loss function 
\begin{flalign}
	\ell_{WF}(\bz):=\frac{1}{4m}\sum_{i=1}^m(|\ba_i^T\bz|^2-y_i^2)^2, \label{eq:WFloss} 
\end{flalign}
via the gradient descent scheme. 

WF is further improved by \emph{truncated Wirtinger flow} (TWF) algorithm proposed in \cite{chen2015solving}, which adopts a Poisson loss function of $|\ba_i^T\bz|^2$, and keeps only well-behaved measurements  based on carefully designed truncation thresholds  for calculating the initial seed and every step of gradient. Such truncation assists to yield linear convergence with certain fixed step size and reduces both the sample complexity to $\cO(n)$ and the convergence time  to $\cO(mn\log 1/\epsilon)$.

It can be observed that WF uses the quadratic loss of $|\ba_i^T\bz|^2$ so that the optimization objective is a {\em smooth} function of $\ba_i^T\bz$ and the gradient step becomes simple. But this comes with a cost of increasing the order of  $\ba_i^T\bz$ to be four in the loss function. In this paper, we adopt the quadratic loss of $|\ba_i^T\bz|$. Although the loss function is not smooth everywhere, it reduces the order of $\ba_i^T\bz$ to be two, and the general curvature can be more amenable to convergence of the gradient method. The goal of this paper is to explore potential advantages of such a nonsmooth lower-order loss function.

Furthermore, incremental/stochastic methods have been proposed to solve Problem \ref{pb:mainproblem}. Specifically, Kaczmarz method for phase retrieval (Kaczmarz-PR)  \cite{wei2015solving, li2015phase} is shown to have excellent  empirical performance, but no global convergence guarantee was established. Incremental truncated Wirtinger flow (ITWF) \cite{kolte2016phase} is a stochastic algorithm developed based on TWF and exhibits linear convergence to the true signal once initialized properly. 
In this paper, we consider the incremental/stochastic version of RWF (IRWF) and compare its performance with Kaczmarz-PR and ITWF, in order to further demonstrate the advantage of the lower-order loss function.

\subsection{Our Contribution}

This paper adopts the following loss function\footnote{The loss function \eqref{eq:rushloss} was also used in \cite{fienup1982phase} to derive a gradient-like update for the phase retrieval problem with Fourier magnitude measurements. However, the focus of this paper is to characterize global convergence guarantee for such an algorithm with appropriate initialization, which was not studied in \cite{fienup1982phase}.}
\begin{flalign}
	\ell(\bz):=\frac{1}{2m} \sum_{i=1}^m \left(|\ba_i^T\bz|-y_i\right)^2. \label{eq:rushloss}
\end{flalign}
Compared to the loss function \eqref{eq:WFloss} in WF that adopts $|\ba_i^T\bz|^2$, the above loss function adopts the absolute value/magnitude $|\ba_i^T\bz|$ and hence has lower-order variables. For such a nonconvex and nonsmooth loss function, we develop a gradient descent-like algorithm, which sets zero for the ``gradient'' component corresponding to nonsmooth samples.  We refer to such an algorithm together with an initialization using a new spectral method (different from that employed in TWF or WF) as {\em reshaped Wirtinger flow} (RWF). We show that the lower-order loss function has great advantage in both statistical and computational efficiency, although scarifying smoothness. In fact, the curvature of such a loss function behaves similarly to that of a least-squares loss function in the neighborhood of global optimums (see Section \ref{sec:gradient}), and hence RWF converges fast. The nonsmoothness does not significantly affect the convergence of the algorithm because only with negligible probability the algorithm encounters nonsmooth points for some samples, which furthermore are set not to contribute to the gradient direction by the algorithm. 


We further exploit the loss function \eqref{eq:rushloss} to design the incremental/stochastic reshaped Wirtinger flow (IRWF), and we show that IRWF also enjoys the advantageous local curvature of such a loss function and achieves excellent statistical and computation performance. In particular, IRWF performs better than other competitive incremental methods (ITWF and Kaczmarz-PR) as well as batch algorithms (RWF, TWF and WF) numerically.

We summarize our main results as follows.
\begin{itemize}
\item Statistically, we show that RWF recovers the true signal with $\cO(n)$ samples, when the design vectors consist of \emph{independently and identically distributed} (i.i.d.) Gaussian entries, which is optimal in the order sense. Thus, even {\em without truncation} in gradient steps (truncation only in initialization stage), RWF improves the  sample complexity $\cO(n\log n)$ of WF, and achieves the same sample complexity as TWF with truncation in gradient step. It is thus more robust to random measurements.

\item Computationally, RWF converges geometrically to the true signal, requiring $\cO(mn\log (1/\epsilon))$ flops to reach $\epsilon-$accuracy. Again, without truncation in gradient steps, RWF improves computational cost $\cO(mn^2\log (1/\epsilon)$ of WF  and achieves the same computational cost as TWF. 

\item Numerically, RWF adopts fixed step size and does not require truncation, which avoids the trouble to set truncation thresholds in practice. It is generally two times faster than TWF and four to six times faster than WF in terms of the number of iterations and time cost.

\item We also show that RWF is robust to bounded additive noise. The estimation error is shown to diminish geometrically to the power of bounded noise up to a certain coefficient. Experiments on Poisson noise further corroborate the stability guarantee.

\item We further propose incremental/stochastic algorithm based on RWF (referred to as IRWF) and show that IRWF converges to the true signal geometrically under an appropriate initialization. More interestingly, we show that randomized Kaczmarz-PR (i.e., Kaczmarz method adapted for phase retrieval \cite{wei2015solving}) can be viewed as IRWF under a specific way of choosing step size, via which we further established geometric convergence guarantee for randomized Kaczmarz-PR. 
\end{itemize}


Compared to WF and TWF, the new form of the gradient step due to nonsmoothness of the loss function, in terms of technical analysis, requires new developments of bounding techniques. On the other hand, our technical proof of performance guarantee is much simpler, because the lower-order loss function allows to bypass higher-order moments of variables and truncation in gradient steps. We also anticipate that such analysis is more easily extendable. 

\subsection{Connection to Related Work}

Along the line of developing nonconvex algorithms with global performance guarantee for the phase retrieval problem, \cite{sujay2013phase} developed alternating minimization algorithm, \cite{candes2015phase,chen2015solving,zhang2016provable,cai2015optimal} developed/studied first-order gradient-like algorithms, and a recent study \cite{sun2016geometric} characterized geometric structure of the nonconvex objective and designed a second-order trust-region algorithm.  This paper is most closely related to \cite{candes2015phase,chen2015solving, sanghavi2016local,
zhang2016provable}, but develops a new gradient-like algorithm based on a lower-order nonsmooth (as well as nonconvex) loss function that yields advantageous statistical/computational efficiency.

Stochastic algorithms are also developed for the phase retrieval problem. \cite{kolte2016phase} studied the incremental truncated Wirtinger flow (ITWF) and showed that ITWF needs much fewer passes of data than TWF to reach the same accuracy. \cite{wei2015solving} adapted the Kaczmarz method to solve the phase retrieval problem and demonstrated its fast empirical convergence. We propose a stochastic algorithm based on RWF (IRWF) and show that it has close connection with Kaczmarz-PR. We also show that IRWF runs faster than ITWF due to the benefit of low-order loss function.

After our work was posted on arXiv, an independent work \cite{wang2016solving} was subsequently posted, which also adopts the same loss function but develops a slightly different algorithm TAF (i.e., truncated amplitude flow). One major difference of our algorithm RWF from TAF is that RWF does not require truncation in gradient loops while TAF employs truncation. Hence, RWF has fewer parameters to tune, and is easier to implement than TAF in practice. Furthermore, RWF demonstrates the performance advantage of adopting a lower-order loss function even without truncation, which cannot be observed from TAF. Moreover, we analyze stochastic algorithm based on new loss function while \cite{wang2016solving} does not.

More generally, various problems have been studied by minimizing nonconvex loss functions. For example, a partial list of these studies include matrix completion \cite{keshavan2010matrix, jain2013low, sun2014guaranteed, hardt2014understanding, de2015global, zheng2016convergence, jin2016provable, ge2016matrix}, low-rank matrix recovery \cite{bhojanapalli2016global, chen2015fast, tu2015low,zheng2015convergent, park2016provable, wei2015guarantees}, robust PCA \cite{netrapalli2014non}, robust tensor decomposition \cite{anandkumar2015tensor}, dictionary learning \cite{arora2015simple,sun2015complete}, community detection \cite{bandeira2016low}, phase synchronization\cite{boumal2016nonconvex}, blind deconvolution \cite{lee2015blind, li2016rapid}, etc. 


For minimizing a general nonconvex nonsmooth objective, various algorithms have been proposed, such as gradient sampling algorithm \cite{burke2005robust, kiwiel2007convergence} and majorization-minimization method \cite{ochs2015iteratively}. These algorithms were often shown to convergence to critical points which may be local minimizers or saddle points, without explicit characterization of convergence rate. In contrast, our algorithm is specifically designed for the phase retrieval problem, and can be shown to converge linearly to global optimum  under appropriate initialization. 


The advantage of nonsmooth loss function exhibiting in our study is analogous in spirit to that of the rectifier activation function (of the form $\max\{0,\cdot\}$) in neural networks. It has been shown that rectified linear unit (\emph{ReLU}) enjoys superb advantage in reducing the training time \cite{krizhevsky2012imagenet} and promoting sparsity \cite{glorot2011deep} over its counterparts of sigmoid and hyperbolic tangent functions, in spite of non-linearity and non-differentiability at zero. Our result in fact also demonstrates that a nonsmooth but simpler loss function yields improved performance.

\subsection{Paper Organization and Notations}
The rest of this paper is organized as follows. Section \ref{sec:guarantee} describes  RWF algorithm in detail and establishes its performance guarantee. 
Section \ref{sec:incremental} introduces the IRWF algorithm and establishes its performance guarantee and compares it with existing stochastic algorithms. Section \ref{sec:algcomp} compares RWF and IRWF with other competitive algorithms numerically. Finally, Section \ref{sec:discussion} concludes the paper with comments on future directions.

Throughout the paper, boldface lowercase letters such as $\ba_i, \bx, \bz$ denote vectors, and boldface capital letters such as $\bA, \bY$ denote matrices. For two matrices, $\bA\preceq\bB$ means that $\bB-\bA$ is positive definite. For a complex matrix or vector, $\bA^*$ and $\bz^*$ denote conjugate transposes of $\bA$ and $\bz$ respectively. For a real matrix or vector, $\bA^T$ and $\bz^T$ denote transposes of $\bA$ and $\bz$ respectively. The indicator function $\bone_A=1$ if the event $A$ is true, and $\bone_A=0$ otherwise.

\section{Reshaped Wirtinger Flow}\label{sec:guarantee}

Consider the problem of recovering a signal $\bx \in \bbR^n$ based on $m$ measurements $y_i$ given by 
\begin{flalign}\label{eq:model}
	y_i=\left|\langle \ba_i,\mathbf{x}\rangle\right|, \quad \text{for }\; i=1,\cdots,m,
\end{flalign}
where $\ba_i \in \bbR^n$ for $i=1,\cdots,m$ are known measurement vectors, independently generated by Gaussian distribution $\cN(0, \bI_{n\times n})$.  It can be observed that if $\bz$ is a solution, i.e., satisfying \eqref{eq:mainproblem}, then $\bz e^{-j\phi}$ is also the solution of the problem. So the recovery is up to a phase difference. Thus, we define the Euclidean distance between two vectors up to a global phase difference \cite{candes2015phase} as, for complex signals,
\begin{flalign}
\dist(\bz,\bx):=\min_{\phi \in [0,2\pi)} \|\bz e^{-j \phi}-\bx\|, \label{eq:defdis}
\end{flalign}
where it is simply $\min \|\bz\pm \bx\|$ for real case. We focus on the real-valued case in analysis, but the algorithm designed below is applicable to the complex-valued case and the case with \emph{coded diffraction pattern} (CDP) as we demonstrate via numerical results in Section \ref{sec:algcomp}.

We design RWF (see Algorithm \ref{alg:rwf}) for solving the above problem, which contains two stages: spectral initialization and gradient loop. Suggested values for parameters are $\alpha_l=1, \alpha_u=5$ and $\mu=0.8$\footnote{For complex Gaussian case, we suggest $\mu=1.2$.}. The scaling parameter in $\lambda_0$ and the conjugate transpose $\ba_i^*$ allow the algorithm readily applicable to complex and CDP cases. We next describe the two stages of the algorithm in detail in Sections \ref{sec:initialization} and \ref{sec:gradient}, respectively, and establish the convergence of the algorithm in Section \ref{sec:convergence}.  Finally, we provide the stability result of RWF in Section \ref{sec:stability}.
\begin{algorithm}[th]
\caption{Reshaped Wirtinger Flow}\label{alg:rwf}

\textbf{Input}: $\by=\{y_i\}_{i=1}^m$, $\{\ba_i\}_{i=1}^m$; \\
\textbf{Parameters:}  Lower and upper thresholds $\alpha_l,\alpha_u$ for  truncation in initialization, stepsize $\mu$;\\
\textbf{Initialization}: Let $\bz^{(0)}=\lambda_0 \tilde{\bz}$, where $\lambda_0=\frac{mn}{\sum_{i=1}^m \|\ba_i\|_1}\cdot \left(\frac{1}{m}\sum_{i=1}^m y_i\right)$ and $\tilde{\bz}$ is the leading eigenvector of
\begin{equation}\label{eq:init_TRWF}
\bY:=\frac{1}{m}\sum_{i=1}^m y_i\ba_i \boldsymbol{a}_i^*\bone_{\{\alpha_l \lambda_0<y_i< \alpha_u \lambda_0\}}.
\end{equation}


 \textbf{Gradient loop}: for $t=0:T-1$ do
  \begin{flalign}\label{eq:loop_FWF}
		\bz^{(t+1)}=\bz^{(t)}- \frac{\mu}{m}\sum_{i=1}^{m}\left(\ba_i^*\bz^{(t)}-y_i\cdot\frac{\ba_i^*\bz^{(t)}}{|\ba_i^*\bz^{(t)}|} \right) \ba_i.
\end{flalign}
\textbf{Output} $\bz^{(T)}$.
\end{algorithm}

\vspace{-2mm}
\subsection{Initialization via Spectral Method}\label{sec:initialization}

Differently from the spectral initialization methods for WF in  \cite{candes2015phase} and for TWF in \cite{chen2015solving}, both of which are based on $|\ba_i^*\bx|^2$, we propose an alternative initialization in Algorithm \ref{alg:rwf} that uses magnitude $|\ba_i^*\bx|$ instead, and truncates samples with both lower and upper thresholds as in \eqref{eq:init_TRWF}. We show that such initialization achieves smaller sample complexity than WF and the same order-level sample complexity as TWF, and furthermore, performs better than both WF and TWF numerically. 

Our initialization consists of estimation of both the norm and direction of $\bx$. The norm estimation of $\bx$ is given by $\lambda_0$ in Algorithm \ref{alg:rwf} with mathematical justification in Appendix \ref{supp:initialization}.
Intuitively, with real Gaussian measurements, the scaling coefficient $\frac{mn}{\sum_{i=1}^m \|\ba_i\|_1}\approx\sqrt{\frac{\pi}{2}}$. 
 Moreover, $y_i=|\ba_i^T\bx|$ are independent sub-Gaussian random variables for $i=1,\ldots,m$ with mean $\sqrt{ \frac{2}{\pi}}\|\bx\|$, and thus $\frac{1}{m}\sum_{i=1}^m y_i\approx\sqrt{ \frac{2}{\pi}}\|\bx\|$. Combining these two facts yields the desired argument.

\begin{figure}[th]
\centering
\includegraphics[width=3.2in]{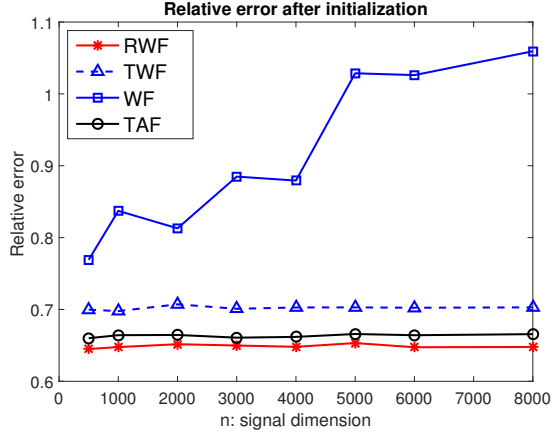}
\caption{Comparison of different initialization methods with $m=6n$ and 50 iterations.}
\label{fig:initGauss}
\vspace{-0.1in}
\end{figure}

The direction of $\bx$ is approximated by the leading eigenvector of $\bY$, because $\bY$ approaches $\mE[\bY]$ by concentration of measure and the leading eigenvector of $\mE[\bY]$ takes the form $c\bx$ for some scalar $c\in \bbR$. We note that \eqref{eq:init_TRWF} involves truncation of samples from both sides, in contrast to truncation only by an upper threshold in \cite{chen2015solving}. This difference is due to the following reason. We note that in high dimension setting, two random vectors are almost perpendicular to each other \cite{cai2013distributions}, which indicates there are considerable amount of $|\ba_i^*\bx|$ with small values, i.e., less than 1. These samples with small values deviate the direction of leading eigenvector of $\bY$ from $\bx$, whose effect cannot be offset and neglected if there are only moderate number of samples. Specifically, \cite{chen2015solving} uses $y'=|\ba_i^*\bx|^2$ to weight the contribution of $\ba_i\ba_i^*$ in $\bY$ and the square power helps to reduce the contribution of bad directions (that samples with small $|\ba_i^*\bx|$ values). In contrast, we use $y_i=|\ba_i^*\bx|$ to weight the contribution of $\ba_i\ba_i^*$ and apply truncation from bellow to filter out bad directions directly. 




We next provide the formal statement of the performance guarantee for the initialization step that we propose. 
\begin{proposition}\label{prop:init}
Fix  $\delta>0$. The initialization step in Algorithm \ref{alg:rwf} yields $\bz^{(0)}$ satisfying $\dist(\bz^{(0)},\bx)\le \delta\|\bx\|$ with probability at least $1-\exp(-cm\epsilon^2)$, if $m>C(\delta,\epsilon) n$, where $c$ is some positive constant and $C$ is a positive number only affected by $\delta$ and $\epsilon$.
\end{proposition}
\begin{proof}
	See Appendix \ref{supp:initialization}. 
\end{proof}
Finally, Figure \ref{fig:initGauss} demonstrates that RWF achieves better initialization accuracy in terms of the relative error $\frac{\dist(\bz^{(0)},\bx)}{\|\bx\|}$ than WF and TWF. Furthermore, we also include the orthonormal promoting initialization method proposed for truncated amplitude flow (TAF) in the independent work \cite{wang2016solving}, in the comparison. It can be seen that our initialization is slightly better.

\subsection{Gradient Loop and Why RWF is Fast}\label{sec:gradient}

The gradient loop of Algorithm \ref{alg:rwf} is based on the loss function \eqref{eq:rushloss}, which is rewritten below
\begin{flalign}
	\ell(\bz):=\frac{1}{2m} \sum_{i=1}^m \left(|\ba_i^T\bz|-y_i\right)^2. \label{eq:rushloss1}
\end{flalign}

We let the update direction be as follows:
\begin{flalign}
	\nabla \ell(\bz):=\frac{1}{m}\sum_{i=1}^m \left(\ba_i^T\bz-y_i\cdot\sgn(\ba_i^T\bz)\right)\ba_i=\frac{1}{m}\sum_{i=1}^m \left(\ba_i^T\bz-y_i\cdot \frac{\ba_i^T\bz}{|\ba_i^T\bz|}\right)\ba_i,\label{eq:rushgrad}
\end{flalign}
where $\sgn(\cdot)$ is the sign function for nonzero arguments. We further set $\sgn(0)=0$ and $\frac{0}{|0|}=0$. In fact, $\nabla \ell(\bz)$ equals the gradient of the loss function \eqref{eq:rushloss1} if $\ba_i^T\bz\neq 0$ for all $i=1, ..., m$. For samples with nonsmooth point, i.e., $\ba_i^T\bz=0$, we adopt Fr\'echet superdifferential \cite{kruger2003frechet} for nonconvex function  to set the corresponding gradient component to be zero (as zero is an element in Fr\'echet superdifferential). With abuse of terminology, we still refer to $\nabla \ell(\bz)$ in \eqref{eq:rushgrad} as ``gradient'' for simplicity, which rather represents the update direction in the gradient loop of Algorithm \ref{alg:rwf}.

We next provide the intuition about why reshaped WF is fast. Suppose that the spectral method sets an initial point in the neighborhood of the ground truth $\bx$. We compare RWF with the following problem of solving $\bx$ from {\em linear equations} $y_i=\langle \ba_i,\mathbf{x}\rangle$ with $y_i$ and $\ba_i$ for $i=1,\ldots,m$ given. In particular, we note that this problem has both magnitude and sign of the measurements observed. In this case, it is natural to use the least-squares loss. Assume that the gradient descent is applied to solve this problem. Then the gradient is given by
\begin{flalign}
	\text{Least-squares gradient:} \quad 	\nabla \ell_{LS}(\bz)=\frac{1}{m}\sum_{i=1}^m \left(\ba_i^T\bz-\ba_i^T\bx\right)\ba_i. \label{eq:LSgrad}
\end{flalign}
We now argue intuitively that the gradient \eqref{eq:rushgrad} of RWF behaves similarly to the
least-squares gradient \eqref{eq:LSgrad}. For each $i$, $y_i\frac{\ba_i^T\bz}{|\ba_i^T\bz|}=|\ba_i^T\bx|\sgn(\ba_i^T\bz)$, and hence the two gradient components are close if $|\ba_i^T\bx|\cdot \sgn(\ba_i^T\bz)$ is viewed as an estimate of $\ba_i^T\bx$. 
The following lemma shows that if $\dist(\bz,\bx)$ is small (guaranteed by initialization), then $\ba_i^T\bz$ has the same sign as $\ba_i^T\bx$ for large $|\ba_i^T\bx|$.
\begin{lemma}\label{clm:samesign}
Let $\ba_i \sim \cN(0, \bI_{n\times n})$. For any given $\bx$ and $\bz$, independent from $\{\ba_i\}_{i=1}^m$, satisfying $\|\bx-\bz\|<\frac{\sqrt{2}-1}{\sqrt{2}}\|\bx\|$, we have
\begin{flalign}\label{eq:samesign}
\bbP\{(\ba_i^T\bx)(\ba_i^T\bz)<0\big|(\ba_i^T\bx)^2= t\|\bx\|^2\}\le\erfc\left(\frac{\sqrt{t}\|\bx\|}{2\|\bz-\bx\|}\right),
\end{flalign}
where  $\erfc(u):=\frac{2}{\sqrt{\pi}}\int_u^\infty \exp(-\tau^2) d\tau$. 
\end{lemma}
\begin{proof}
	See Appendix \ref{supp:samesign}.
\end{proof}
It is easy to observe in \eqref{eq:samesign} that large $\ba_i^T\bx$ is likely to have the same sign as $\ba_i^T\bz$ so that the corresponding gradient components in \eqref{eq:rushgrad} and \eqref{eq:LSgrad} are likely equal, whereas small $\ba_i^T\bx$ may have different sign as $\ba_i^T\bz$ but contributes less to the gradient. Hence, overall the two gradients \eqref{eq:rushgrad} and \eqref{eq:LSgrad} should be close to each other with a large probability. 


This fact can be further verified numerically. Figure \ref{fig:leastsquare} illustrates that RWF takes almost the {\em same} number of iterations for recovering a signal (with only magnitude information) as the least-squares gradient descent method for recovering a signal (with both magnitude and sign information). 
\begin{figure}[th]
\centering 
\subfigure[Convergence behavior]{\includegraphics[width=3in]{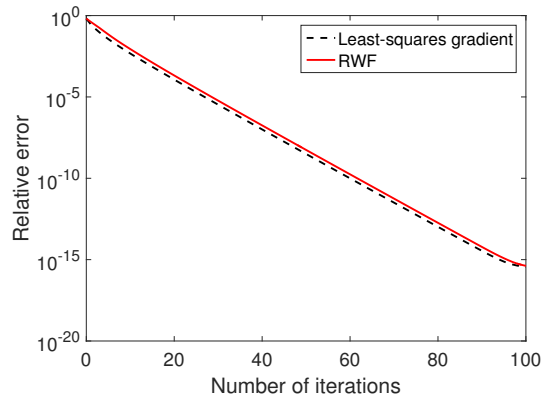}
\label{fig:leastsquare}
}
\caption{Comparison of convergence behavior between RWF and least-squares gradient descent with the same initialization. Parameters $n=1000$, $m=6n$, step size $\mu=0.8$.}
\end{figure}

\begin{figure}[th]
	\centering
\subfigure[Quadratic surface]{\includegraphics[width=1.8in]{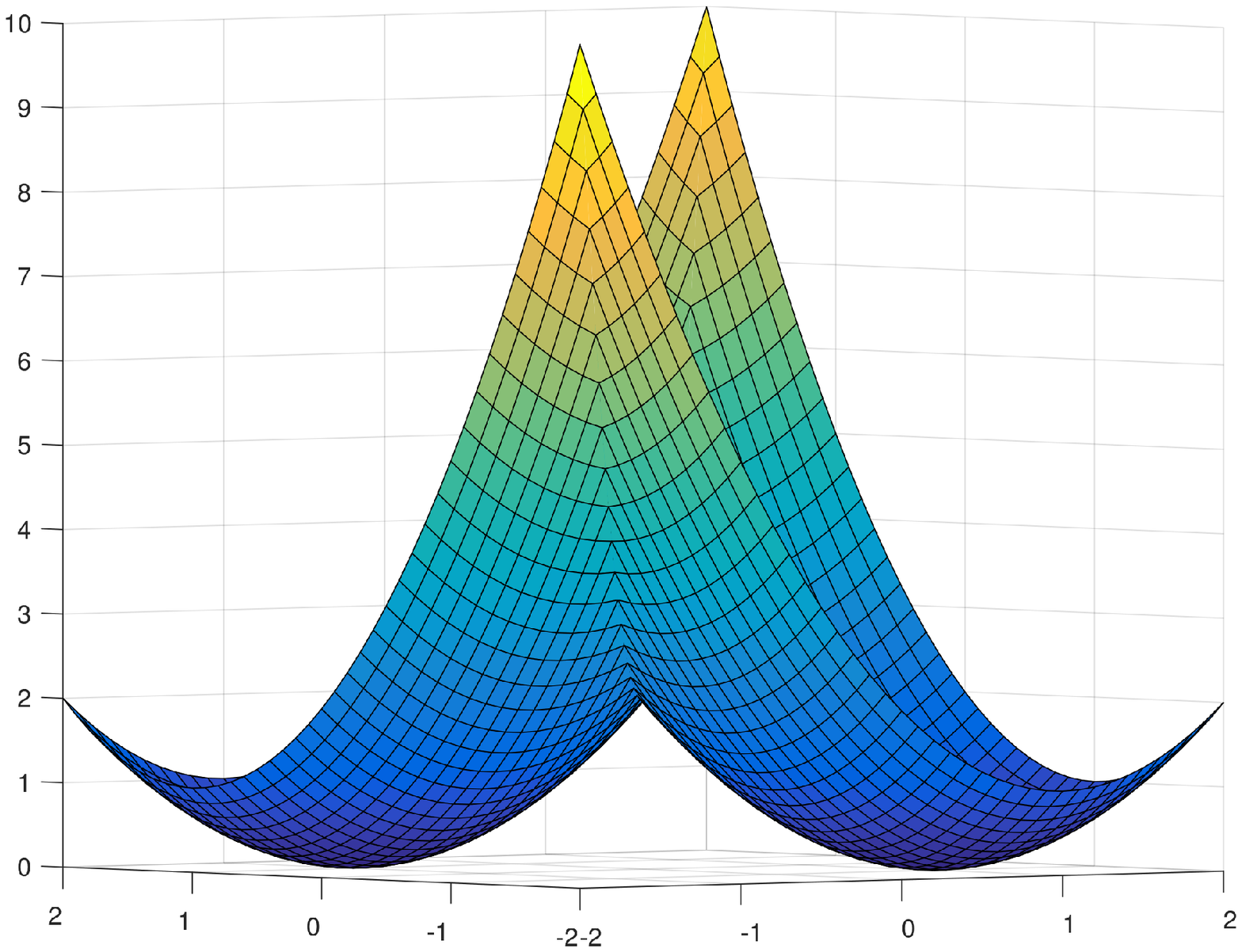}
\label{fig:quadraticsurf}
}
\hfil
\subfigure[Expected loss of RWF]{
\includegraphics[width=1.8in]{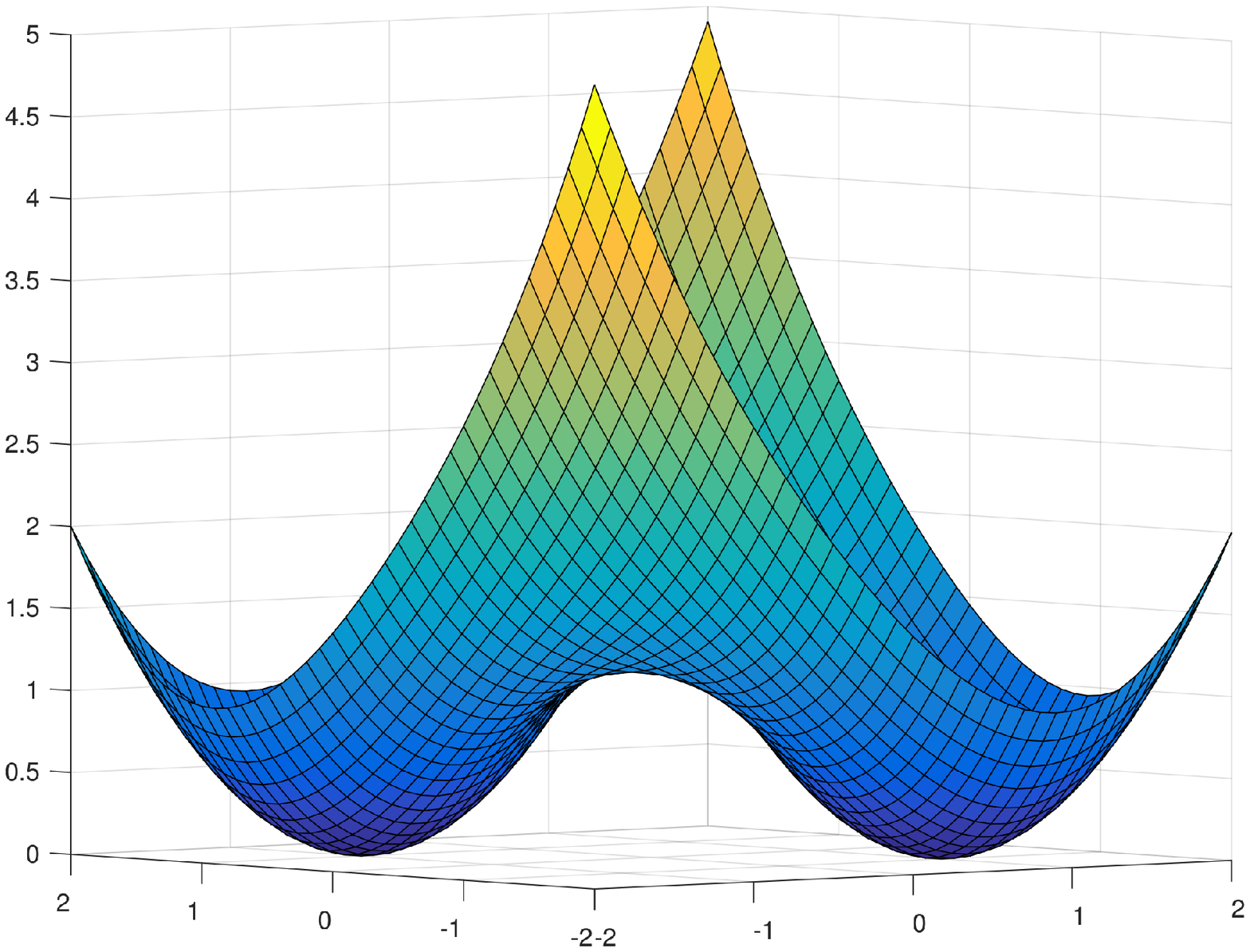}
\label{fig:rwfasym}
}
\hfil
\subfigure[Expected loss of WF]{
\includegraphics[width=1.8in]{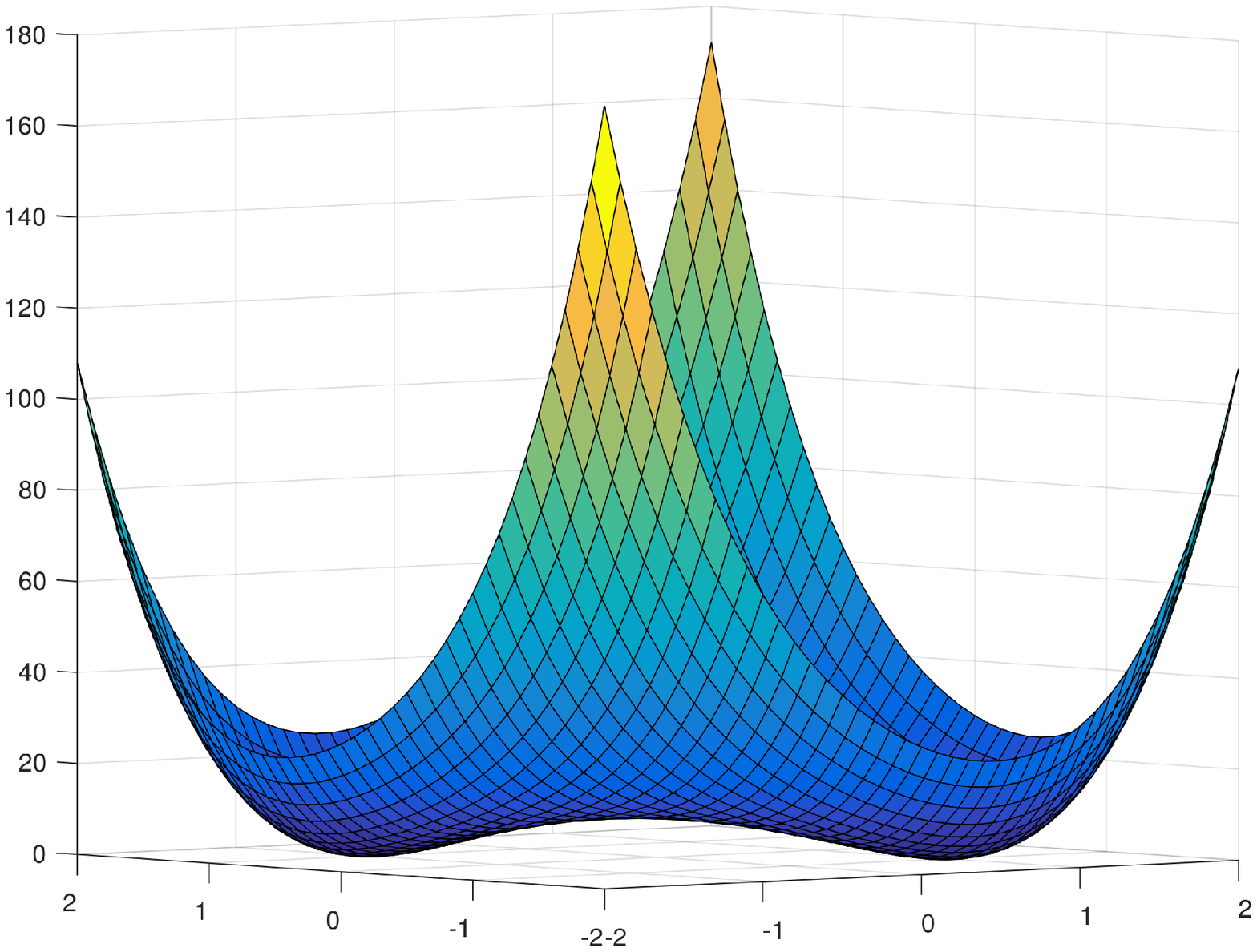}
\label{fig:wfasym}}
\caption{(a) Surface of quadratic function $f(z)=\min \{(\bz-\bx)^T(\bz-\bx), (\bz+\bx)^T(\bz+\bx)\}$ with $\bx=[1 -1]^T$. (b) Expected loss function of RWF for $\bx=[1 -1]^T$. (c) Expected loss function of WF for $\bx=[1 -1]^T$.}
\label{fig:losssurface}
\end{figure}

Furthermore, Figure \ref{fig:losssurface} illustrates a quadratic function $f(\cdot)$ (Figure \ref{fig:quadraticsurf}), the expected loss surface of RWF (Figure \ref{fig:rwfasym}, and see Appendix \ref{supp:whyfast} for expression), and the expected loss surface for WF  (Figure \ref{fig:wfasym}, and see Appendix \ref{supp:whyfast} for expression). It can be seen that the loss of RWF rather than the loss of WF has a similar curvature to the quadratic function around the global optimums, which justifies its better performance than WF.

\subsection{Geometric Convergence of RWF}\label{sec:convergence}
We characterize the convergence of RWF in the following theorem.
\begin{theorem}\label{th:mainthm}
Consider the problem of solving any given $\bx\in \bbR^n$ from a system of equations \eqref{eq:model} with Gaussian measurement vectors. 
There exist some universal constants $\mu_0>0$ ($\mu_0$ can be set as $0.8$ in practice), $0<\rho,\nu<1$ and $c_0,c_1,c_2>0$ such that if $m\ge c_0 n$ and $\mu<\mu_0$, then with probability at least $1-c_1\exp(-c_2 m)$, Algorithm \ref{alg:rwf} yields  
\begin{flalign}
	\dist(\bz^{(t)},\bx)\le \nu(1-\rho)^t \|\bx\|, \quad \forall t\in \bbN.
\end{flalign}
\end{theorem}
\begin{proof}[Outline of the Proof]
We outline the proof here with details relegated to Appendix \ref{supp:convergence}. Compared to WF and TWF, our proof requires new development of bounding techniques to deal with nonsmoothness, but is much simpler due to the lower-order loss function that RWF relies on. 

We first introduce a global phase notation for real case as follows:
\begin{flalign}
	\Phi(z): =\begin{cases}
0, \;\; &\text{ if } \|\bz-\bx\|\le \|\bz+\bx\|,\\
\pi, & \text{otherwise.}
\end{cases}
\end{flalign}
For the sake of simplicity, we let $\bz$ be $e^{-j\Phi(\bz)}\bz$, which indicates that $\bz$ is always in the neighborhood of $\bx$.

Here, the central idea is to show that within the neighborhood of global optimums, RWF satisfies the {\em Regularity Condition} $\mathsf{RC}(\mu,\lambda,c)$ \cite{chen2015solving}, i.e.,
\begin{flalign}
\left\langle \nabla \ell(\bz), \bz-\bx \right\rangle \ge \frac{\mu}{2} \left\|\nabla \ell(\bz)\right\|^2+\frac{\lambda}{2} \|\bz-\bx\|^2 \label{eq:RC}
\end{flalign}
for all $\bz$ obeying $\|\bz-\bx\|\le c \|\bx\|$, where $0<c<1$ is some constant. Then, as shown in \cite{chen2015solving}, once the initialization lands into this neighborhood, geometric convergence can be guaranteed, i.e., 
\begin{flalign}
	\dist^2\left(\bz+\mu\nabla\ell(\bz),\bx\right)\le (1-\mu\lambda)\dist^2(\bz,\bx),
\end{flalign}
for any $\bz$ with $\|\bz-\bx\|\le c\|\bx\|$.


Lemmas \ref{lem:lemma1} and \ref{lem:lemma2} in Appendix \ref{supp:convergence} yield that 
\[\left\langle \nabla \ell(\bz), \bz-\bx \right\rangle\ge (1-0.26-2\epsilon)\|\bz-\bx\|^2=(0.74-2\epsilon)\|\bz-\bx\|^2.\]
And Lemma \ref{lem:lemma3} in Appendix \ref{supp:convergence} further yields that 
\begin{flalign}
	\|\nabla \ell(\bz)\|\le (1+\delta)\cdot 2\|\bz-\bx\|. \label{eq:lemma3upper}
\end{flalign}
Therefore, the above two bounds imply that Regularity Condition \eqref{eq:RC} holds for $\mu$ and $\lambda$ satisfying
\begin{flalign} \label{eq:mulambda}
0.74-2\epsilon\ge \frac{\mu}{2}\cdot 4(1+\delta)^2+\frac{\lambda}{2}.
\end{flalign} 
\end{proof}

We note that \eqref{eq:mulambda} implies an upper bound $\mu\le \frac{0.74}{2}=0.37$, by taking $\epsilon$ and $\delta$ to be sufficiently small. This suggests a range to set the step size in Algorithm \ref{alg:rwf}. However, in practice, $\mu$ can be set much larger than such a bound, say $0.8$, while still keeping the algorithm convergent. This is because the coefficients in the proof are set for convenience of proof rather than being tightly chosen.

Theorem \ref{th:mainthm} indicates that RWF recovers the true signal with $\cO(n)$ samples, which is order-level optimal. Such an algorithm improves the  sample complexity $\cO(n\log n)$ of WF. Furthermore, RWF does not require truncation of weak samples in the gradient step to achieve the same sample complexity as TWF. This is mainly because RWF benefits from the lower-order loss function given in \eqref{eq:rushloss1}, the curvature of which behaves similarly to the least-squares loss function locally as we explain in Section \ref{sec:gradient}.


Theorem \ref{th:mainthm} also suggests that RWF converges geometrically at a constant step size. To reach $\epsilon-$accuracy, it requires computational cost of $\cO(mn\log 1/\epsilon)$ flops, which is better than WF ($\cO(mn^2\log (1/\epsilon)$). Furthermore, it does not require truncation in gradient steps to reach the same computational cost as TWF. Numerically, as we demonstrate in Section \ref{sec:algcomp}, RWF is two times faster than TWF and four to six times faster than WF in terms of both iteration count and time cost in various examples.


\subsection{Stability to Bounded Noise}\label{sec:stability}

We have established that RWF guarantees exact recovery with geometric convergence for noise-free case. We now study RWF in the presence of noise. Suppose the measurements are corrupted by bounded noise, and are given by
\begin{flalign}
y_i=|\ba_i^T\bx|+w_i, \quad 1\le i \le m, 	\label{eq:noisymodel}
\end{flalign}
where $\{w_i\}_{i=1}^m$ denote the additive noise. Then the following theorem shows that RWF is robust under such noise corruption.
\begin{theorem}
Consider the model \eqref{eq:noisymodel}. Suppose that the measurement vectors are independently Gaussian, i.e., $\ba_i \sim \cN(0,\bI)$ for $1\le i \le m$, and the noise is bounded, i.e., $\|\bw\|/\sqrt{m}\le c\|\bx\|$. Then there exist some universal constants $\mu_0>0$ ($\mu_0$ can be set as $0.8$ in practice), $0<\rho<1$ and $c_0,c_1,c_2>0$ such that if $m\ge c_0 n$ and $\mu<\mu_0$, then with probability at least $1-c_1\exp(-c_2 m)$, Algorithm \ref{alg:rwf} yields
\begin{flalign}
	\dist(\bz^{(t)},\bx)\lesssim \frac{\|\bw\|}{\sqrt{m}}+(1-\rho)^t\|\bx\|, \quad \forall t\in \bbN,
\end{flalign}
for some $\rho\in (0,1)$. 
\end{theorem}
\begin{proof}
	See Appendix \ref{sec:proofofstability}.
\end{proof}
 The numerical result under the Poisson noise model in Section \ref{sec:algcomp} further corroborates the stability of RWF.

\section{Incremental Reshaped Wirtinger Flow}\label{sec:incremental}
In large sample size and online scenarios, stochastic algorithm is preferred due to its potential advantage of fast convergence and low computational complexity. Thus, in this section, we develop the stochastic algorithm based on RWF, named incremental reshaped Wirtinger flow (IRWF). We show that IRWF guarantees exact recovery with linear convergence rate. We further draw the connection between IRWF and the Kaczmarz-PR algorithm recently developed for phase retrieval \cite{wei2015solving, li2015phase, chi2016kaczmarz}.

\subsection{IRWF: Algorithm and Convergence}\label{sec:irwf}
We describe IRWF in Algorithm \ref{alg:irwf}. More specifically, IRWF applies the same initialization step  as in RWF, but calculates each gradient update using only one sample selected randomly.

\begin{algorithm}[th]
\caption{Incremental Reshaped Wirtinger Flow (IRWF)}\label{alg:irwf}

\textbf{Input}: $\by=\{y_i\}_{i=1}^m$, $\{\ba_i\}_{i=1}^m$; \\
\textbf{Initialization}: Same as in RWF (Algorithm \ref{alg:rwf}); \\

 \textbf{Gradient loop}: for $t=0:T-1$ do\\
 Choose $i_t$ uniformly at random from $\{1,2,\ldots, m\}$, and let
  \begin{flalign}
		\bz^{(t+1)}=\bz^{(t)}- \mu\left(\ba_{i_t}^*\bz^{(t)}-y_{i_t}\cdot\frac{\ba_{i_t}^*\bz^{(t)}}{|\ba_{i_t}^*\bz^{(t)}|} \right) \ba_{i_t}, \label{eq:incrementalupdate}
\end{flalign}
\textbf{Output} $\bz^{(T)}$.
\end{algorithm}

We characterize the convergence of IRWF in the following theorem.
\begin{theorem}\label{th:incremental}
Consider the problem of solving any given $\bx\in \bbR^n$ from a system of equations \eqref{eq:model} with Gaussian measurement vectors. 
There exist some universal constants $0<\rho,\rho_0, \nu<1$ and $c_0,c_1,c_2>0$ such that if $m\ge c_0 n$ and $\mu=\rho_0/n$, then with probability at least $1-c_1\exp(-c_2 m)$, Algorithm \ref{alg:irwf} yields  
\begin{flalign}
	\mE_{\mathcal{I}^t}\left[\dist^2(\bz^{(t)},\bx)\right]\le \nu\left(1-\frac{\rho}{n}\right)^t \|\bx\|^2, \quad \forall t\in \bbN,		\label{eq:incremental}
\end{flalign}
where  $\mE_{\cI^t}[\cdot]$ denotes the expectation with respect to algorithm randomness $\cI^t=\{i_1,i_2,\ldots, i_t\}$ conditioned on the high probability event of random measurements $\{\ba_i\}_{i=1}^m$.
\end{theorem}

 We suggest the step size $\rho_0=1$ in practice.
\begin{proof}
The proof is relegated to Appendix \ref{sec:proofofincremental}, which uses technical lemmas established for proving Theorem \ref{th:mainthm}.
\end{proof}
 
Theorem \ref{th:incremental} establishes that IRWF achieves linear convergence to the global optimum. For general objectives, it is not anticipated that incremental/stochastic first-order method achieves linear convergence  due to the variance of stochastic gradient. However, for our specific problem, the variance of stochastic gradient reduces as the estimate approaches the true solution, and hence a fixed step size can be employed and linear convergence can be established. Such a result is also established in \cite{kolte2016phase} for the stochastic algorithm based on TWF (ITWF). We comment further on comparison between our algorithm and ITWF in Section \ref{sec:compITWF}.  Another explanation may be due to the fact \cite{moulines2011non, needell2016stochastic} that stochastic gradient method yields linear convergence to the minimizer $\bx_\star$ when the objective $F(\bx)=\sum_{i}f_i(\bx)$ is a smooth and strongly convex function and  $\bx_\star$ minimizes all components $f_i(\bx)$. This may also hold for our objective \eqref{eq:rushloss} whose summands share a same minimizer, although it is neither convex nor smooth. 

\subsection{Minibatch IRWF: Algorithm and Convergence}\label{sec:birwf}
In oder to fully exploit the processing throughput of CPU/GPU, we develop a minibatch version of IRWF, described in Algorithm \ref{alg:birwf}.  The minibatch IRWF applies the same initialization step  as in RWF, but uses a minibatch of data for each gradient update in contrast to IRWF that uses only a single sample. 

\begin{algorithm}[th]
\caption{Minibatch Incremetnal Reshaped Wirtinger Flow (minibatch IRWF)}\label{alg:birwf}

\textbf{Input}: $\by=\{y_i\}_{i=1}^m$, $\{\ba_i\}_{i=1}^m$; \\
\textbf{Initialization}: Same as in RWF (Algorithm \ref{alg:rwf}); \\

 \textbf{Gradient loop}: for $t=0:T-1$ do\\
 Choose $\Gamma_t$ uniformly at random from the subsets of  $\{1,2,\ldots, m\}$ with cardinality $k$, and let
  \begin{flalign}
	\bz^{(t+1)}=\bz^{(t)}-\mu\cdot\bA_{\Gamma_t}^* \left(\bA_{\Gamma_t}\bz^{(t)}-\by_{\Gamma_t}\odot\phase (\bA_{\Gamma_t}\bz^{(t)})\right), \label{eq:birwfUpdate}
\end{flalign}
where $\bA_{\Gamma_t}$ is a matrix stacking $\ba_i^*$ for $i\in \Gamma_t$ as its rows,  $\by_{\Gamma_t}$ is a vector stacking $y_i$ for $i\in \Gamma_t$ as its elements, $\odot$ denotes element-wise product, and $\phase(\bz)$ denotes a phase vector of $\bz$.

\textbf{Output} $\bz^{(T)}$.
\end{algorithm}

We characterize the convergence of minibatch IRWF in the following theorem.
\begin{theorem}\label{th:blockincremental}
Consider the problem of solving any given $\bx\in \bbR^n$ from a system of equations \eqref{eq:model} with Gaussian measurement vectors. 
There exist some universal constants $0<\rho,\rho_0,\nu<1$ and $c_0,c_1,c_2>0$ such that if $m\ge c_0 n$ and $\mu=\rho_0/n$, then with probability at least $1-c_1\exp(-c_2 m)$, Algorithm \ref{alg:birwf} yields  
\begin{flalign}
	\mE_{\Gamma^t}\left[\dist^2(\bz^{(t)},\bx)\right]\le \nu\left(1-\frac{k\rho}{n}\right)^t \|\bx\|^2, \quad \forall t\in \bbN,		\label{eq:blockincremental}
\end{flalign}
where $\mE_{\Gamma^t}[\cdot]$ denotes the expectation with respect to algorithm randomness $\Gamma^t=\{\Gamma_1,\Gamma_2,\ldots, \Gamma_t\}$ conditioned on the high probability event of random measurements $\{\ba_i\}_{i=1}^m$.
\end{theorem}
\begin{proof}
See Appendix \ref{suppsub:blockincremental}.
\end{proof}
We suggest that $\rho_0=1$ in practice. 
%

\subsection{Connection to Kaczmarz Method for Phase Retrieval}\label{sec:kaczmarz}

\emph{Kaczmarz method} was originally developed for solving the linear equation systems $\bA\bx=\bb$ \cite{kaczmarz1937angenaherte}. In \cite{wei2015solving}, it was adapted to solve phase retrieval problem, which we refer to as \emph{Kaczmarz-PR}. It has been demonstrated in \cite{wei2015solving} that Kaczmarz-PR exhibits better empirical performance than error reduction (ER) \cite{gerchberg1972practical, fienup1982phase} and Wirtinger flow (WF) \cite{candes2015phase}. 
However, theoretical guarantee of Kaczmarz-PR has not been well established yet although Kaczmarz method for least-squares problem achieves linear convergence guarantee  \cite{strohmer2009randomized, zouzias2013randomized}.  For instance, \cite{wei2015solving} obtained a bound on the estimation error which can be as large as the signal energy no matter how many iterations are taken. \cite{li2015phase} requires infinite number of samples to establish the asymptotic convergence. 

In this section, we draw the connection between IRWF and Kaczmarz-PR, and thus the theoretical guarantee of Kaczmarz-PR can be established by adapting that of IRWF. This is analogous to the connection established in \cite{needell2016stochastic} between Kaczmarz method and stochastic gradient method when solving the least-squares problem. Here, the connection is interesting because RWF rather than WF and TWF connects to Kaczmarz-PR due to the lower-order loss function that RWF adopts.

To be more specific, the Kaczmarz-PR (Algorithm 3 in \cite{wei2015solving}) employs the following update rule
\begin{flalign}
	\bz^{(t+1)}=\bz^{(t)}- \frac{1}{\|\ba_{i_t}\|^2}\left(\ba_{i_t}^*\bz^{(t)}-y_{i_t}\cdot\frac{\ba_{i_t}^*\bz^{(t)}}{|\ba_{i_t}^*\bz^{(t)}|} \right) \ba_{i_t}, 	\label{eq:simpleKaczmarz}
\end{flalign}
where $i_t$ is selected either in a deterministic manner or randomly. We focus on the randomized Kaczmarz-PR where $i_t$ is selected uniformly at random.

Comparing \eqref{eq:simpleKaczmarz} and \eqref{eq:incrementalupdate}, the update rule of randomized Kaczmarz is a special case of  IRWF with step size $\mu$ replaced by $\frac{1}{\|\ba_{i_t}\|^2}$. Moreover, these two update rules are close if $\mu$ is set as suggested, i.e., $\mu=\frac{1}{n}$, because $\|\ba_{i_t}\|^2$ concentrates around $n$ by law of large numbers. As we demonstrate in empirical results (see Table \ref{tab:algcomp}), these two methods have similar performance as anticipated. Thus, following the convergence result  Theorem \ref{th:incremental} for IRWF, we have the convergence guarantee for the randomized Kaczmarz-PR as follows.
\begin{theorem}\label{cor:kaczmarz}
Assume the measurement vectors are independent and each $\ba_i\sim \cN(0,\bI)$.
There exist some universal constants $0<\rho<1$ and $c_0,c_1,c_2>0$ such that if $m\ge c_0 n$, then with probability at least $1-c_1 m \exp(-c_2 n)$, the randomized Kaczmarz update rule \eqref{eq:simpleKaczmarz} yields  
\begin{flalign}
	\mE_{i_t}\left[\dist^2(\bz^{(t+1)},\bx)\right]\le \left(1-\frac{\rho}{n}\right)\cdot \dist^2(\bz^{(t)},\bx)		\label{eq:kaczmarzresult}
\end{flalign}
holds for all $\bz^{(t)}$ satisfying $\frac{\dist(\bz^{(t)},\bx)}{\|\bz\|}\le \frac{1}{10}$.
\end{theorem}
\begin{proof}
	See Appendix \ref{supp:subsec:kaczmarz}.
\end{proof}
The above theorem implies that once the estimate $\bz^{(t)}$ enters the neighborhood of true solutions (often referred as to \emph{basin of attraction}), the error diminishes geometrically by each update in expectation.

Furthermore, \cite{wei2015solving} also provided a \emph{block} Kaczmarz-PR (similar to the minibatch version), whose update rule is given by
\begin{flalign}
\bz^{(t+1)}=\bz^{(t)}-\bA_{\Gamma_t}^\dagger \left(\bA_{\Gamma_t}\bz^{(t)}-\by_{\Gamma_t}\odot\phase (\bA_{\Gamma_t}\bz^{(t)})\right), \label{eq:blockKaczmarz}
\end{flalign} 
where $\odot$ is element-wise product, $\phase(\ba)$ denotes a phase vector of $\ba$ (stacking phase of each element of $\ba$ together), $\bA$ is a matrix with each row being measurement vector $\ba_i^*$, and $\Gamma_t$ is a selected block at iterate $t$ containing row indices. Moreover, $\dagger$ represents \emph{Moore}-{\em Penrose pseudoinverse}, which can be computed as follows:
\begin{flalign}
\bA^\dagger =\begin{cases}
(\bA^*\bA)^{-1} \bA^*, \quad \text{if } \bA \text{ has linearly independent columns};\\
\bA^*(\bA\bA^*)^{-1}, \quad \text{if } \bA \text{ has linearly independent rows}.
\end{cases}
\end{flalign}
Comparing \eqref{eq:blockKaczmarz} and the minibatch IRWF update in \eqref{eq:birwfUpdate}, these two update rules are similar to each other if $\bA_{\Gamma_t}\bA_{\Gamma_t}^*$  approaches $\frac{n}{\rho_0}\bI_{|\Gamma_t|}$. For the case with Gaussian measurements, $\bA_{\Gamma_t}$ has linearly independent rows with high probability if $|\Gamma_t|\le n$ and hence $\bA_{\Gamma_t}\bA_{\Gamma_t}^*$ is not far from $n\bI_{|\Gamma_t|}$. Our empirical results (see Table \ref{tab:algcomp}) further suggest similar convergence rate for these two algorithms with the same block size.

Next, we argue that for the CDP setting, block Kaczmarz-PR is the same as the minibatch IRWF with $\mu=\frac{1}{n}$. The CDP measurements are collected in the following form
\begin{flalign}
	\by^{(l)}=|\bF\bD^{(l)}\bx|, \quad 1\le l\le L, \label{eq:CDPmodel}
\end{flalign}
where $\bF$ represents the discrete Fourier transform (DFT) matrix, and $\bD^{(l)}$ is a diagonal matrix (mask).
We choose the block size $|\Gamma_t|$ to be $n$, the dimension of the signal, for the convenience of Fourier transform.  Then $\bA_{\Gamma_t}$ becomes a Fourier transform composed with $\bD^{(l)}$ (mask effect) and $\bA_{\Gamma_t}^*$ becomes $\bD^{(l)*}$ multiplied by inverse Fourier transform. Therefore, $(\bA_{\Gamma_t}\bA_{\Gamma_t}^*)=\bI$ if the diagonal elements of $\bD^{(l)}$ have unit magnitude.  Taking the step size $\mu=1$, the two algorithms are identical. 

On the other hand, since the block Kaczmarz-PR needs to calculate the matrix inverse or to solve an inverse problem, the block size cannot be too large. However, minibatch IRWF works well for a wide range of batch sizes which can even vary with signal dimension $n$ as long as a batch of data is loadable into memory. 

\subsection{Comparison with Incremental Truncated Wirtinger Flow (ITWF)}\label{sec:compITWF}
Recently, \cite{kolte2016phase} designed and analyzed an incremental algorithm based on TWF, which is referred to as ITWF. More specifically, ITWF employs the same initialization procedure as TWF and randomly chooses one sample for gradient update as follows.
\begin{flalign}
\text{Step $t$: Sample $i_t$ uniformly at random from } \{1,2,\ldots,m\}, \; \text{ and } \nn\\
\quad \bz^{(t+1)}=\bz^{(t)}-\frac{\rho_0}{n}\cdot \frac{|\ba_i^T\bz|^2-y_i^2}{\ba_i^T\bz}\ba_i \bone_{\cE_{1,t}^{i_t}\cap \cE_3^{i_t}},
\end{flalign}
where $\bone_{\cE_{1,t}^{i_t}\cap \cE_3^{i_t}}$ is the truncation rule determined by two events $\cE_{1,t}^{i_t}$ and $\cE_3^{i_t}$. 
Compared to ITWF developed based on TWF, our IRWF uses lower-order variable $|\ba_i^T\bz|$ rather than $|\ba_i^T\bz|^2$ used in ITWF. Moreover, IRWF does not employ any truncation in gradient loops and hence has fewer parameters to tune, which is easier to implement in practice. 

\cite{kolte2016phase} proved that ITWF converges linearly to the true signal as long as $m/n$ (sample size/ signal dimension) is large enough. Compared to ITWF, IRWF also achieves the same linear convergence, but runs faster than ITWF numerically as demonstrated in Section \ref{sec:algcomp}. The proof of IRWF requires different bounding techniques, but is simpler than ITWF due to the lower-order loss function that RWF adopts and the avoidance of truncation in the gradient update.
 
\section{Numerical Results}\label{sec:algcomp}

In this section, we demonstrate the numerical efficiency of RWF and IRWF by comparing their performances with other competitive algorithms. Our experiments are conducted not only for real Gaussian case but also for complex Gaussian and the CDP cases. All the experiments are implemented in Matlab 2015b and carried out on a computer equipped with Intel Core i7 3.4GHz CPU and  12GB RAM.

We first compare the sample complexity of RWF and IRWF with those of TWF and WF via the empirical successful recovery rate versus the number of measurements. 
For RWF,  we follow Algorithm \ref{alg:rwf} with suggested parameters. For IRWF,  we adopt a block size 64 for efficiency and set the step size $\rho_0=1$. For WF, TWF, we use the codes provided in the original papers with the suggested parameters. For ITWF,  we also adopt a block size 64 and set the step size $\rho_0=0.6$ (optimal step size). 
We conduct the experiment for real Gaussian, complex Gaussian and CDP cases respectively. For real  and complex cases, we set the signal dimension $n$ to be 1000, and set the ratio $m/n$ to take values from $2$ to $6$ by a step size $0.1$. For each $m$, we run $100$ trials and count the number of successful trials. For each trial, we run a fixed number of iterations/passes $T=1000$ for all algorithms. A trial is declared to be successful if  $\bz^{(T)}$, the output of the algorithm, satisfies $\dist(\bz^{(T)},\bx)/\|\bx\|\le 10^{-5}$. For the real Gaussian case, we generate signal $\bx\sim \cN(\bzero,\bI_{n\times n})$, and the measurement vectors $\ba_i\sim \cN(\bzero,\bI_{n\times n})$ i.i.d. for $i=1,\ldots, m$.  For the complex Gaussian case, we generate signal $\bx\sim \cN(0,\bI_{n\times n})+j\cN(0,\bI_{n\times n})$ and measurements $\ba_i\sim  \frac{1}{2}\cN(0,\bI_{n\times n})+j\frac{1}{2}\cN(0,\bI_{n\times n})$ i.i.d. for $i=1, \ldots, m$. For the CDP case \eqref{eq:CDPmodel}, we set $n=1024$ for convenience of FFT and $m/n=L=1,2,\ldots, 8$. All other settings are the same as those for the real case.
\begin{figure}[th]
\centering 
\subfigure[Real Gaussian case]{
\includegraphics[width=3in]{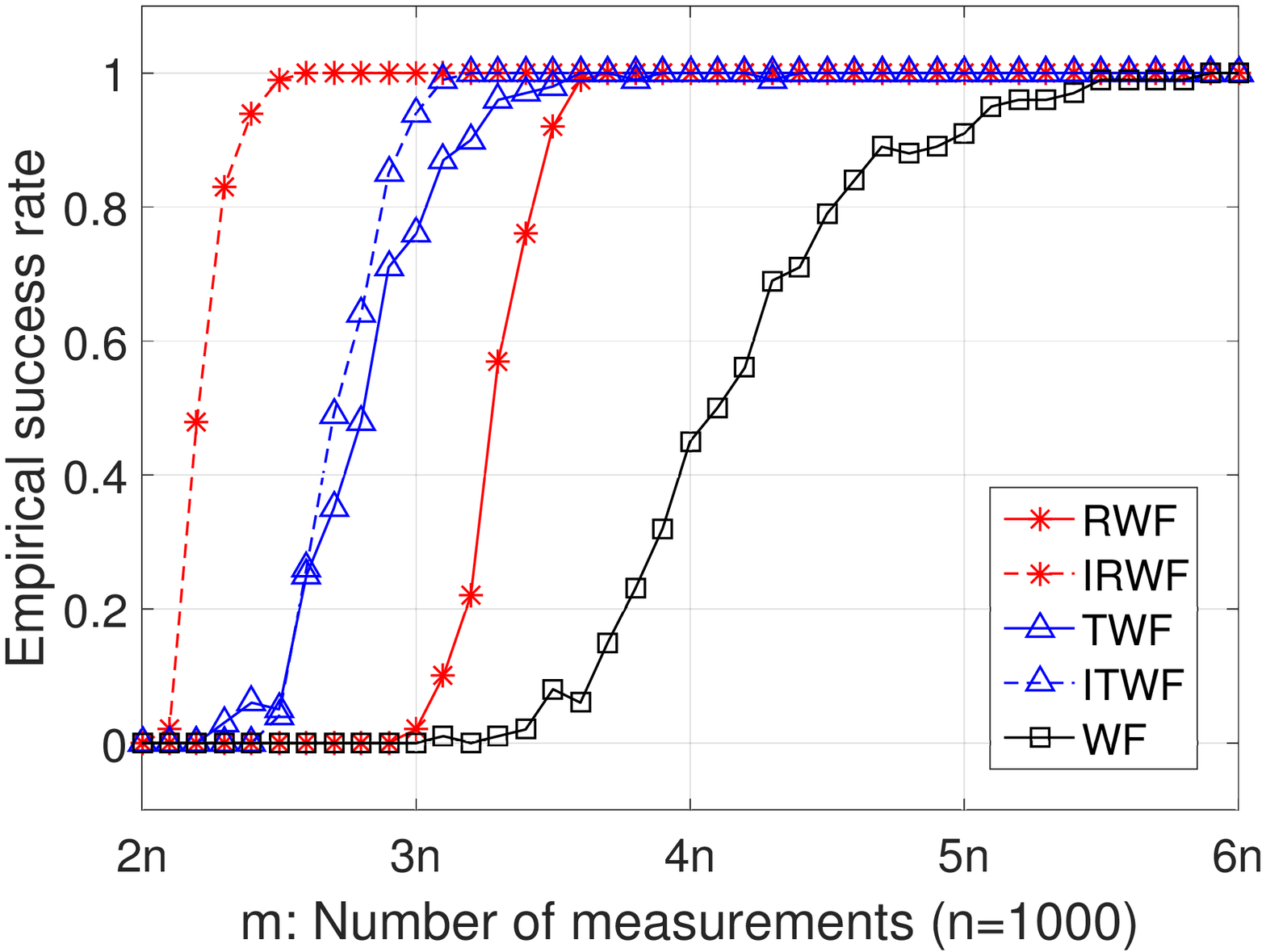}
}
\hfil
\subfigure[Complex Gaussian case]{
\includegraphics[width=3in]{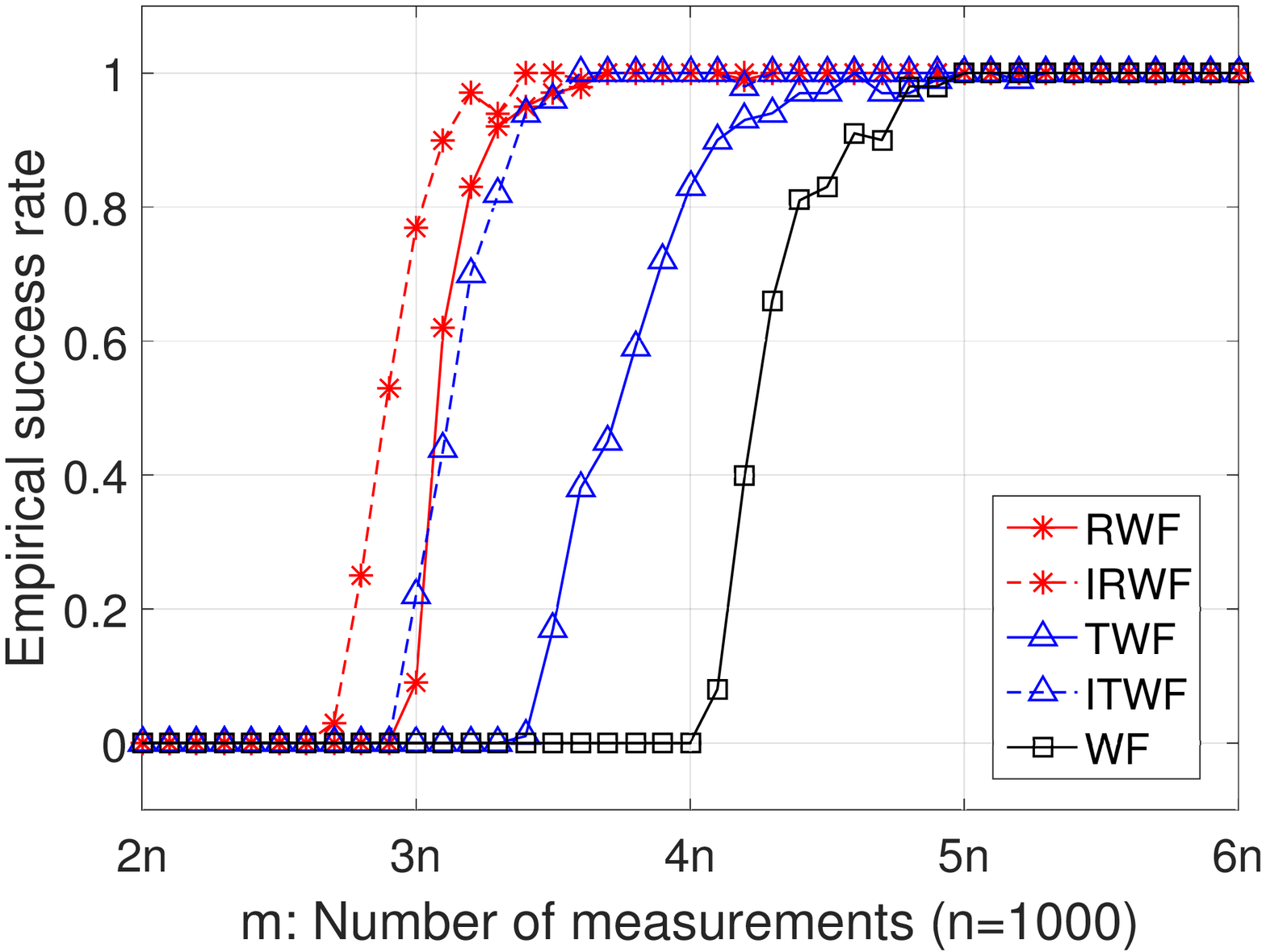}
}
\subfigure[CDP case]{
\includegraphics[width=3in]{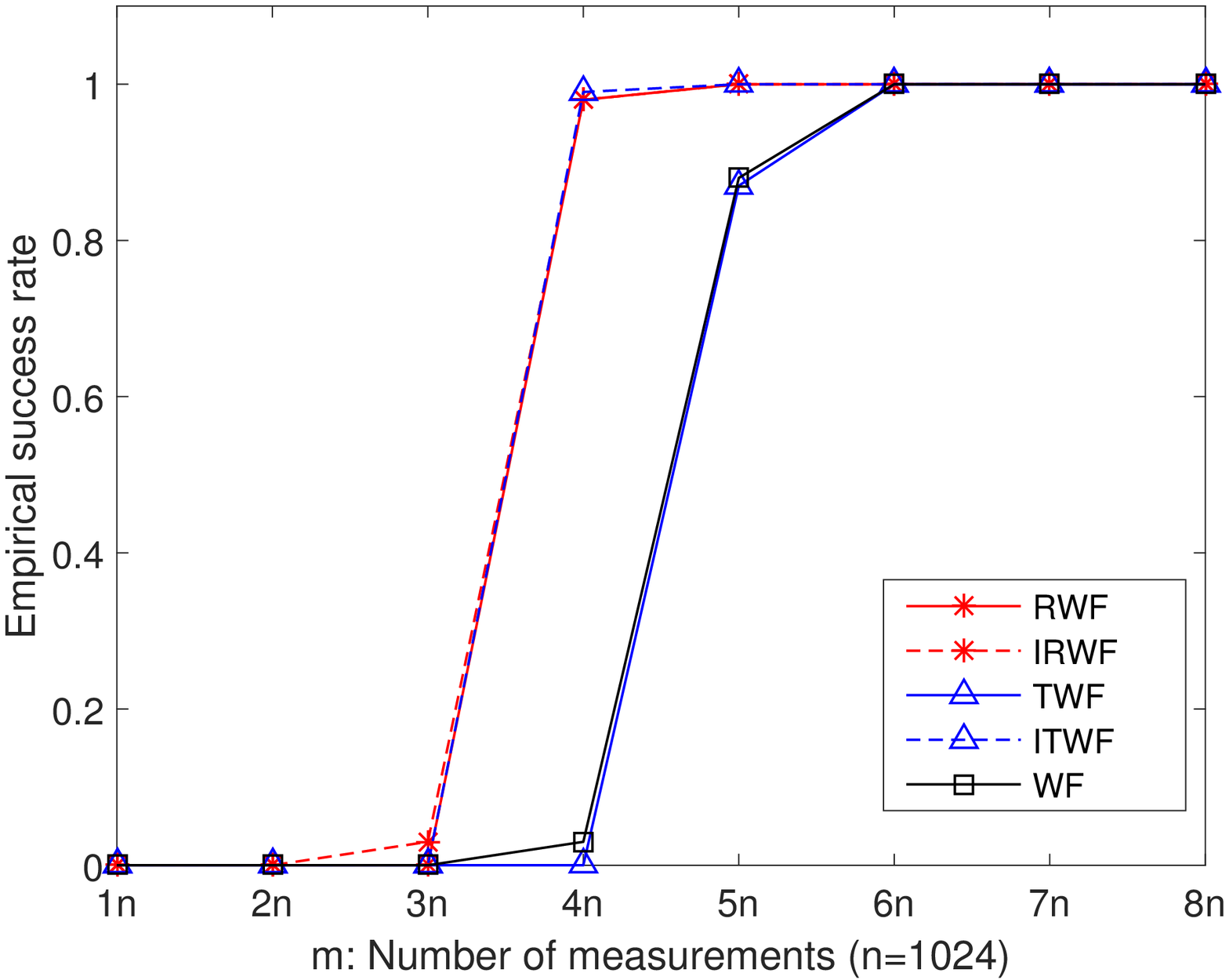}
}
\vspace{-0.1in}
\caption{Comparison of sample complexity among RWF, IRWF, TWF, ITWF and WF.} 
\vspace{-4mm}
\label{fig:phaseTransition}
\end{figure}

Figure~\ref{fig:phaseTransition} plots the fraction of successful trials out of 100 trials for all algorithms, with respect to $m$. It can be further seen that IRWF exhibits the best sample complexity for all three cases, which is close to the theoretical limits \cite{xu2015minimal}. It can be seen that the two incremental methods (IRWF and ITWF) outperform batch methods (RWF, TWF and WF). This can be due to the inherent noise in incremental methods helps to escape bad local minimums, which is extremely helpful in the regime of  small number of samples where local minimums do exist near the global ones. Comparing among the three batch methods (RWF, TWF and WF), it can be seen that although RWF outperforms only WF (not TWF) for the real Gaussian case, it outperforms both WF and TWF for complex Gaussian and CDP cases. An intuitive explanation for the real case is that a substantial number of samples with small $|\ba_i^T\bz|$ can deviate gradient so that truncation indeed helps to stabilize the algorithm if the number of measurements is not large. Furthermore, RWF exhibits \emph{sharper} transition than TWF and WF. 

\begin{table}[t]
  \caption{Comparison of iteration count and time cost among algorithms $(n=1000, m=8n)$}
  \label{tab:algcomp}
  \centering
  \begin{tabular}{l|lll|ll}
    \toprule
   & &\multicolumn{2}{c|}{Real Gaussian} &\multicolumn{2}{c}{Complex Gaussian}            \\
     & &\#passses &time(s)  & \# passes & time(s) \\
     \midrule
      &RWF& 72 &\textbf{0.52} &177 &\textbf{4.81} \\[3pt]
     Batch&TWF & 182& 1.30&484 &13.5 \\[3pt]
     methods&WF & 217&2.22&922&24.9 \\[3pt]
	&AltMinPhase & \textbf{6} &0.94 &\textbf{157} &91.8 \\[3pt]
	\midrule
	
      &IRWF &9 &3.15 &21 &11.8 \\[3pt]
     &minibatch IRWF (64) &9 &\textbf{0.28} &\textbf{21} &\textbf{1.53} \\[3pt]
     Incremental&minibatch ITWF (64) & 15& 0.72&28.6 &3.28 \\[3pt]
     methods&Kaczmarz-PR & 9&3.71& 21 &13.2 \\[3pt]
     &block Kaczmarz-PR (64)& \textbf{8} &0.45 &21 &3.22 \\[3pt] 
         \bottomrule
  \end{tabular}
\end{table}
We next compare the convergence rate of RWF, IRWF with those of TWF, ITWF, WF and AltMinPhase. 
We run all of the algorithms with suggested parameter settings in the original codes. 
We generate signal and measurements in the same way as those in the first experiment with $n=1024, m=8n$. All algorithms are seeded with RWF initialization. In Table \ref{tab:algcomp}, we list the number of passes and time cost for those algorithms to achieve the relative error of $10^{-14}$ averaged over 10 trials. For incremental methods, one update passes $k$ samples and one pass amounts to $m/k$ updates. 
Clearly, IRWF with minibatch size 64 runs fastest for both real and complex cases. Moreover, among batch (deterministic) algorithms, RWF takes much less number of passes as well as running much faster than TWF and WF. Although RWF takes more iterations than AltMinPhase, it runs much faster than AltMinPhase due to the fact that each iteration of AltMinPhase needs to solve a least-squares problem that takes much longer time than a simple gradient update in RWF.



We also compare the performance of the above algorithms on the recovery of a real image from the Fourier intensity measurements (two dimensional CDP case). The image (see Figure \ref{fig:galaxy}) is the Milky Way Galaxy with resolution $1920\times1080$. Table \ref{tab:galaxy} lists the number of passes and the time cost of the above six algorithms to achieve the relative error of $10^{-15}$ for one R/G/B channel. All algorithms are seeded with RWF initialization. To explore the advantage of FFT, we run the incremental/stochastic methods with minibatch size of the one R/G/B channel. We note that with such a minibatch size IRWF is equivalent to block Kaczmarz-PR from the discussion in Section \ref{sec:kaczmarz}. It can be seen that in general, the incremental/stochastic methods (IRWF/Kaczmarz-PR and ITWF) run faster than the batch methods (RWF, TWF, WF and AltMinPhase). Moreover, among batch methods, RWF outperforms other three algorithms  in both number of passes and the computational time cost. In particular, RWF runs two times faster than TWF and six times faster than WF in terms of both the number of iterations and computational time cost. 
\begin{figure}[th]
	\centering 
	\includegraphics[width=5.5in]{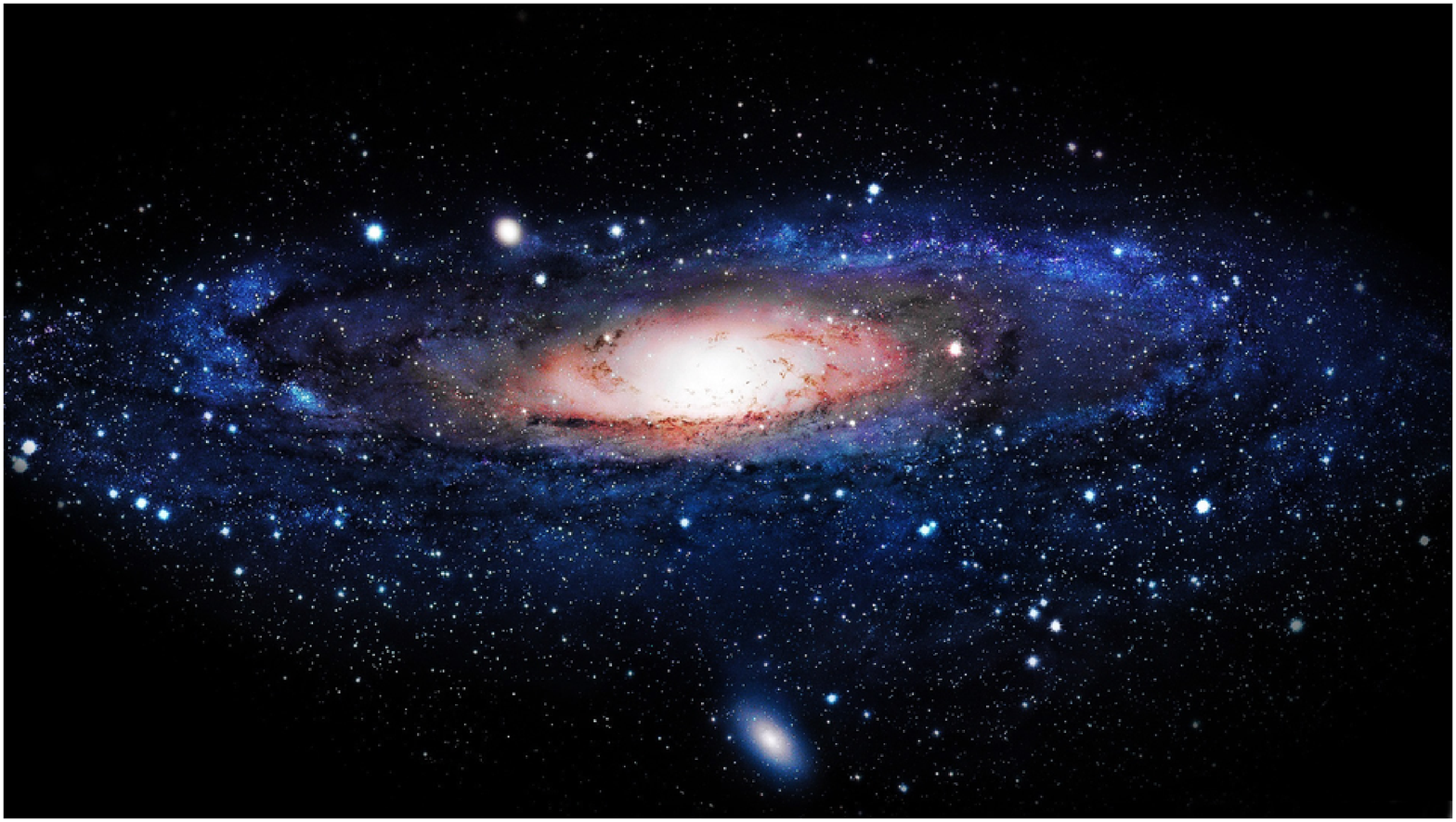}
	\caption{Milky way Galaxy.}
	\label{fig:galaxy}
\end{figure}
\begin{table}[th]
  \caption{Comparison of iterations and time cost among algorithms on Galaxy image (shown in Figure \ref{fig:galaxy})}
  \label{tab:galaxy}
  \centering
  \begin{tabular}{llllllll}
    \toprule
    &Algorithms     & RWF  &IRWF/Kaczmarz-PR   & TWF & ITWF& WF & AltMinPhase\\
    \midrule
   $L=6$ &\#passes & 140&\textbf{24}  &410&41&fail&230      \\
   & time cost(s)    &110  & \textbf{21.2} &406 &43&fail&167    \\
   \midrule
   $L=12$ &\#passes & 70&\textbf{8}  &190&12&315&120      \\
   & time cost(s)    &107  & \textbf{13.7} &363.6 &25.9&426&171   \\
    \bottomrule
  \end{tabular}
\end{table}

We next demonstrate the robustness of RWF to noise corruption and compare it with TWF. We consider the phase retrieval problem in imaging applications, where random Poisson noises are often used to model the sensor and electronic noise \cite{fogel2013phase}. Specifically, the noisy measurements of intensity can be expressed as y$_i=\sqrt{\alpha\cdot \text{Poisson}\left(|\ba_i^T\bx|^2/\alpha\right)},\quad \text{for } i=1,2,...m$ where $\alpha$ denotes the level of input noise, and $\text{Poisson}(\lambda)$ denotes a random sample generated by the Poisson distribution with mean $\lambda$. 
It can be observed from Figure \ref{fig:poissonnoise} that RWF performs better than TWF in terms of recovery accuracy under different noise levels.
\begin{figure}[!th]
\centering 
\includegraphics[width=3.2in]{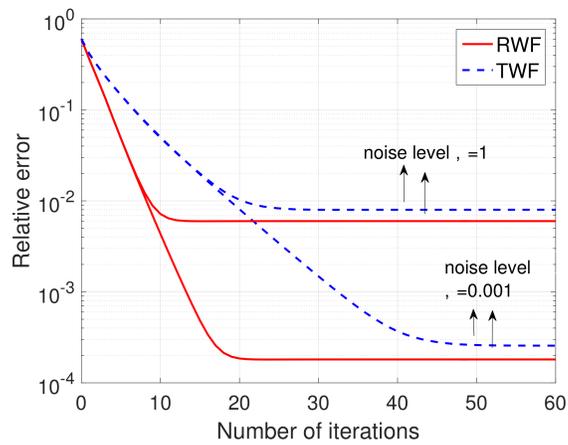}
\vspace{-0.1in}
\caption{Comparison of relative error under Poisson noises between RWF and TWF.} 
\label{fig:poissonnoise}
\end{figure}

\vspace{-4mm}
\section{Conclusion}\label{sec:discussion}

In this paper, we proposed RWF and its stochastic version IRWF to recover a signal from a quadratic systems of equations, based on a {\em nonconvex and nonsmooth} quadratic loss function of absolute values of measurements. This loss function sacrifices the smoothness but enjoys advantages in statistical and computational efficiency. It has potential to be extended in various scenarios. One interesting direction is to extend such an algorithm to exploit signal structures (e.g., non-negativity, sparsity, etc) to assist the recovery. The lower-order loss function may offer great simplicity to prove performance guarantee in such cases. 

Another interesting direction is to study the convergence of algorithms from random initialization. In the regime of large sample size ($m\gg n$), the empirical loss surface approaches the asymptotic loss (Figure \ref{fig:rwfasym}) and hence has no spurious local minimums. Due to \cite{lee2016gradient}, it is conceivable that gradient descent converges from random starting point. Similar phenomenons have been observed in \cite{sun2016geometric, ge2016matrix}. However, under moderate number of measurements ($m<10n$), authentic local minimums do exist and often locate not far from the global ones. In this regime, the batch gradient method often fails with random initialization. As always believed, stochastic algorithms are efficient in escaping bad local minimums or saddle points in nonconvex optimization because of the inherent noise \cite{ge2015escaping,de2015global}. We observe numerically that IRWF and block IRWF from random starting point still converge to global minimum even with very small sample size which is close to the theoretical limits \cite{xu2015minimal}.    It is of interest to analyze theoretically that stochastic methods escape these local minimums (not just saddle points) efficiently. 

\clearpage

\onecolumn
\appendix
\noindent {\LARGE \textbf{Appendix}}
\vspace{1cm}

We first introduce some notations here. We let $\cA:\bbR^{n\times n}\mapsto \bbR^m$ be a linear map
\begin{flalign*}
	\bM \in \bbR^{n\times n} \mapsto \cA(\bM):=\{\ba_i^T\bM\ba_i\}_{1\le i\le m}.
\end{flalign*}	
We let $\|\cdot\|_1$ and $\|\cdot\|$ denote the $l_1$ norm and $l_2$ norm of a vector, respectively. Moreover, let $\|\cdot\|_F$ and $\|\cdot\|$ denote the Frobenius norm and the spectral norm of a matrix, respectively. We note that the constants $c, C, c_0, c_1, c_2$ may be different from line to line, for the sake of notational simplicity.

\section{Proof of Proposition \ref{prop:init}: Performance Guarantee for Initialization}\label{supp:initialization}

Compared to the proof for TWF \cite{chen2015solving}, this proof has new technical developments to address the magnitude measurements and truncation from both sides. 

We first estimate the norm of $\bx$ as
\begin{flalign}
\lambda_0=\frac{mn}{\sum_{i=1}^m \|\ba_i\|_1}\cdot \left(\frac{1}{m}\sum_{i=1}^m y_i\right). \label{eq:normest}
\end{flalign}
Since $\ba_i\sim \cN(0, \bI_{n\times n})$, by Hoeffding-type inequality, it can be shown that
\begin{flalign}
	\left|\frac{\sum_{i=1}^m \|\ba_i\|_1}{mn}-\sqrt{\frac{2}{\pi}}\right|<\frac{\epsilon}{3} \label{eq:scalingest}
\end{flalign}
holds with probability at least $1-2\exp(-c_1mn\epsilon^2)$ for some constant $c_1>0$.

Moreover, given $\bx$, $y_i$'s are independent sub-Gaussian random variables. Thus, by Hoeffding-type inequality, it can be shown that
\begin{flalign}
\left|\sqrt{\frac{\pi}{2}}\left(\frac{1}{m}\sum_{i=1}^m y_i\right)-\|\bx\|\right|<\frac{\epsilon}{3}\|\bx\| \label{eq:normestbound}
\end{flalign}
holds with probability at least $1-2\exp(-c_1m\epsilon^2)$ for some constant $c_1>0$.

On the event $E_1=\{ \text{both } \eqref{eq:scalingest} \text{ and }  \eqref{eq:normestbound} \text{ hold}\}$, it can be argued that 
\begin{flalign}
	\left|\lambda_0-\|\bx\|\right|<\epsilon\|\bx\|.
\end{flalign}

Without loss of generality, we let $\|\bx\|=1$. Then on the event $E_1$, the truncation function satisfies the following bounds
\begin{flalign*}
\bone_{\{\alpha_l(1+\epsilon)<|\ba_i^T\bx|< \alpha_u (1-\epsilon)\}}\le\bone_{\{\alpha_l \lambda_0<y_i< \alpha_u \lambda_0\}}\le \bone_{\{\alpha_l(1-\epsilon)<|\ba_i^T\bx|< \alpha_u(1 +\epsilon)\}}.
\end{flalign*}
Thus, by defining
\begin{flalign*}
\bY_1:=&\frac{1}{m} \sum \ba_i\ba_i^T |\ba_i^T\bx| \bone_{\{\alpha_l(1+\epsilon)<|\ba_i^T\bx|< \alpha_u (1-\epsilon)\}} \nn \\
\bY_2:=&\frac{1}{m} \sum \ba_i\ba_i^T |\ba_i^T\bx| \bone_{\{\alpha_l(1-\epsilon)<|\ba_i^T\bx|< \alpha_u(1 +\epsilon)\}},
\end{flalign*}
we have $\bY_1\prec \bY\prec \bY_2$. We further compute the expectations of $\bY_1$ and $\bY_2$ and obtain
\begin{flalign}
	\mE[\bY_1]=(\beta_1 \bx\bx^T+\beta_2 \bI), \quad \mE[\bY_2]=(\beta_3 \bx\bx^T +\beta_4 \bI),
\end{flalign}
where 
\begin{flalign}
& \beta_1:=\mE[|\xi|^3\bone_{\{\alpha_l(1+\epsilon)<|\xi|< \alpha_u(1 -\epsilon)\}}]-\mE[|\xi|\bone_{\{\alpha_l(1+\epsilon)<|\xi|< \alpha_u(1 -\epsilon)\}}], \nn\\
& \beta_2:=\mE[|\xi|\bone_{\{\alpha_l(1+\epsilon)<|\xi|< \alpha_u (1-\epsilon)\}}] \nn \\
& \beta_3:=\mE[|\xi|^3 \bone_{\{\alpha_l(1-\epsilon)<|\xi|< \alpha_u(1 +\epsilon)\}}]-\mE[|\xi| \bone_{\{\alpha_l(1-\epsilon)<|\xi|< \alpha_u(1 +\epsilon)\}}], \nn \\
& \beta_4:=\mE[|\xi| \bone_{\{\alpha_l(1-\epsilon)<|\xi|< \alpha_u (1+\epsilon)\}}] \nn
\end{flalign}
where $\xi\sim\cN(0,1)$. For given $\alpha_l$ and $ \alpha_u$, small value of $\epsilon$ yields arbitrarily close $\beta_1$ and $\beta_3$, as well as arbitrarily close $\beta_2$ and $ \beta_4$. For example, taking $\alpha_l=1, \alpha_u=5$ and $\epsilon=0.01$, we have $\beta_1=0.9678, \beta_2=0.4791, \beta_3=0.9678, \beta_4=0.4888$.


Now applying standard results on random matrices with non-isotropic sub-Gaussian rows \cite[equation (5.26)]{Vershynin2012} and noticing that $\ba_i\ba_i^T|\ba_i^T\bx|\bone_{\{\alpha_l(1+\epsilon)<|\ba_i^T\bx|< \alpha_u (1-\epsilon)\}}$ can be rewritten as $\bb_i\bb_i^T$ for sub-Gaussian vector $\bb_i:=\ba_i\sqrt{|\ba_i^T\bx|}\bone_{\{\alpha_l(1+\epsilon)<|\ba_i^T\bx|< \alpha_u (1-\epsilon)\}}$, one can derive
\begin{flalign}
	\|\bY_1-\mE[\bY_1]\|\le \delta, \quad \|\bY_2-\mE[\bY_2]\|\le \delta
\end{flalign}
with probability $1-4\exp(-c_1(\delta)m)$ for some positive $c_1$ which is only affected by $\delta$, provided that $m/n$ exceeds a certain constant. Furthermore, when $\epsilon$ is sufficiently small, one further has $\|\mE[\bY_1]-\mE[\bY_2]\|\le \delta$.
Combining the above facts together, one can show that
\begin{flalign}
\|\bY-(\beta_1 \bx\bx^T+\beta_2 \bI) \|\le 3\delta.
\end{flalign}
Let $\tilde{\bz}^{(0)}$ be the normalized leading eigenvector of $\bY$. Following the arguments in \cite[Section 7.8]{candes2015phase} and taking $\delta$ and $\epsilon$ to be sufficiently small, one has
\begin{flalign}
	\dist(\tilde{\bz}^{(0)},\bx)\le \tilde{\delta},
\end{flalign}
for a given $\tilde{\delta}>0$, as long as $m/n$ exceeds a certain constant.

\section{Supporting Arguments for Section~\ref{sec:gradient}}\label{supp:whyfast}

\subsection{Expectation of loss functions}\label{supp:exploss}
The expectation of the loss function \eqref{eq:WFloss} of WF  is given by \cite{sun2016geometric} as
\begin{flalign}
\mE[ \ell_{WF}(\bz)]&=\frac{3}{4}\|\bx\|^4+\frac{3}{4}\|\bz\|^4-\frac{1}{2}\|\bx\|^2\|\bz\|^2-|\bz^T\bx|^2 \label{eq:WFlossasym}.
\end{flalign}

We next show that the expectation of the loss function \eqref{eq:rushloss} of RWF  has the following form:
\begin{flalign}
\mE[ \ell(\bz)]&=\frac{1}{2}\|\bx\|^2+\frac{1}{2}\|\bz\|^2-\|\bx\|\|\bz\|\cdot\mE\left[ \frac{|\ba_i^T\bz|}{\|\bz\|}\cdot\frac{|\ba_i^T\bx|}{\|\bx\|}\right] \label{eq:rushlossasym},
\end{flalign}
where
\begin{flalign}
\mE\left[ \frac{|\ba_i^T\bz|}{\|\bz\|}\cdot\frac{|\ba_i^T\bx|}{\|\bx\|}\right]=\begin{cases}
\frac{(1-\rho^2)^{3/2}}{\pi }\int_0^\infty t(e^{\rho t}+e^{-\rho t}) K_0(t) dt,\quad& \text{if } |\rho|<1;\\
1, \quad &\text{if } |\rho|=1;
\end{cases} \label{eq:expectoftwogaussian}
\end{flalign}
where $\rho=\frac{\bz^T\bx}{\|\bx\|\|\bz\|}$ and $K_0(\cdot)$ is the modified Bessel function of the second kind.

In order to derive \eqref{eq:expectoftwogaussian}, we first define 
\begin{flalign*}
	u:=\frac{\ba_i^T\bz}{\|\bz\|} \text{  and  } v:=\frac{\ba_i^T\bx}{\|\bx\|},
\end{flalign*}
and it suffices to drive $\mE [|uv|]$. Note that $(u,v)\sim \cN(0,\Sigma)$, where
\begin{flalign*}
	\Sigma=\left[\begin{array}{cc}
1 &\rho\\
\rho &1
\end{array}\right], \quad \text{and}\quad \rho=\frac{\bz^T\bx}{\|\bx\|\|\bz\|}.
\end{flalign*}
Following \cite{donahue1964products}, the density function of $u\cdot v$ is given by
\begin{flalign*}
	\phi_{uv}(x)=\frac{1}{\pi \sqrt{1-\rho^2}}\exp\left(\frac{\rho x}{1-\rho^2}\right) K_0\left(\frac{|x|}{1-\rho^2}\right), \quad x\neq 0.
\end{flalign*}
Thus, the density of $|uv|$ is given by
\begin{flalign}\label{eq:psi_rho}
	\psi_{|uv|}(x)=\frac{1}{\pi \sqrt{1-\rho^2}}\left[\exp\left(\frac{\rho x}{1-\rho^2}\right) +\exp\left(-\frac{\rho x}{1-\rho^2}\right) \right]K_0\left(\frac{|x|}{1-\rho^2}\right),\quad  x> 0,
\end{flalign}
for $|\rho|<1$. Therefore, if $|\rho|<1$, then
\begin{flalign*}
	\mE [|uv|]&=\int_0^\infty x\cdot \psi_\rho(x) dx\\
	&=\int_0^\infty x\cdot \frac{1}{\pi \sqrt{1-\rho^2}}\left[\exp\left(\frac{\rho x}{1-\rho^2}\right) +\exp\left(-\frac{\rho x}{1-\rho^2}\right) \right]K_0\left(\frac{|x|}{1-\rho^2}\right) dx\\
	&=\frac{(1-\rho^2)^{3/2}}{\pi }\int_0^\infty t(e^{\rho t}+e^{-\rho t}) K_0(t) dt  
\end{flalign*}
where the last step follows by changing variables.

If $|\rho|=1$, then $|uv|$ becomes a $\chi_1^2$ random variable, with the density
\begin{flalign*}
	\psi_{|uv|}(x)=\frac{1}{\sqrt{2\pi}}x^{-1/2}\exp(-x/2), \quad x> 0,
\end{flalign*}
and hence $\mE[|uv|]=1$.



\subsection{Proof of Lemma \ref{clm:samesign}}\label{supp:samesign}

Let $\ba(1)$ denote the first element of a generic vector $\ba$, and $\ba(-1)$ denote the remaining vector of $\ba$ after eliminating the first element. Let $\bU_x$ be an orthonormal matrix with first row being $\bx^T/\|\bx\|$, $\tilde{\ba}_i=\bU_x \ba_i$, and $\tilde{\bh}=\bU_x \bh$.   Similarly define $\bU_{\tilde{h}(-1)}$ and let $\tilde{\bb}_i=\bU_{\tilde{h}(-1)}\tilde{\ba}_i(-1)$. Then $\tilde{\ba}_i(1)$ and $\tilde{\bb}_i(1)$ are independent standard Gaussian random variables. 

We evaluate the conditional probability as follows.
\begin{flalign*}
\bbP&\{(\ba_i^T\bx)(\ba_i^T\bz)<0\big|(\ba_i^T\bx)^2= t\|\bx\|^2\}\\
&=\bbP\{t\|\bx\|^2+(\ba_i^T\bx)(\ba_i^T\bh)<0\big|(\ba_i^T\bx)^2= t\|\bx\|^2\} \quad\text{due to } \bz=\bx+\bh\\
&\le\bbP\{t\|\bx\|^2-\sqrt{t}\|\bx\||\ba_i^T\bh|<0\big|(\ba_i^T\bx)^2= t\|\bx\|^2\}\\
&=\bbP\{|\ba_i^T\bh|>\sqrt{t}\|\bx\|\big|(\ba_i^T\bx)^2= t\|\bx\|^2\}\\
&=\bbP\{|\tilde{\ba}_i(1) \tilde{\bh}(1)+\tilde{\ba}_i(-1)^T\tilde{\bh}(-1)|>\sqrt{t}\|\bx\|\big||\tilde{\ba}_i(1)|= \sqrt{t}\} \quad\quad \text{orthogonal transformation } U_x\\
&\le \bbP\{|\tilde{\ba}_i(1) \tilde{\bh}(1)|+|\tilde{\ba}_i(-1)^T\tilde{\bh}(-1)|>\sqrt{t}\|\bx\|\big||\tilde{\ba}_i(1)|= \sqrt{t}\}\\
&=\bbP\left\{|\tilde{\ba}_i(-1)^T\tilde{\bh}(-1)|>\sqrt{t}\left(\|\bx\|-\frac{|\bh^T\bx|}{\|\bx\|}\right)\Big||\tilde{\ba}_i(1)|= \sqrt{t}\right\}  \quad \text{due to } \tilde{\bh}(1)=\frac{\bh^T\bx}{\|\bx\|}\\
&=\bbP\left\{|\bb_i(1)|\cdot \sqrt{\|\bh\|^2-\frac{(\bh^T\bx)^2}{\|\bx\|^2}}>\sqrt{t}\left(\|\bx\|-\frac{|\bh^T\bx|}{\|\bx\|}\right)\right\} \quad \quad \text{due to } \bb=U_{\tilde{\bh}(-1)}\tilde{\ba}_i(-1)\\
&=\bbP\left\{|\bb_i(1)| >\sqrt{t}\cdot \frac{\|\bx\|}{\|\bh\|}\left(1-\frac{|\bh^T\bx|}{\|\bx\|^2}\right)\big/\sqrt{1-(\bh^T\bx)^2/(\|\bh\|^2\|\bx\|^2)}\right\} \\
&\le\bbP\left\{|\bb_i(1)| >\sqrt{t}\cdot \frac{\|\bx\|}{\|\bh\|}\left(1-\frac{|\bh^T\bx|}{\|\bx\|^2}\right)\right\}\\
&\le\bbP\left\{|\bb_i(1)| >\sqrt{t}\cdot \left(\frac{\|\bx\|}{\|\bh\|}-1\right)\right\} \quad \quad \text{by Cauchy-Schwartz inequality}\\
&\le\bbP\left\{|\bb_i(1)/\sqrt{2}| >\frac{\sqrt{t}}{\sqrt{2}} \cdot \left(\frac{\|\bx\|}{\|\bh\|}-1\right)\right\}\\
&\le\erfc\left(\frac{\sqrt{t}\|\bx\|}{2\|\bh\|}\right)
\end{flalign*}
where $\erfc(z):=\frac{2}{\sqrt{\pi}}\int_z^\infty \exp(-t^2) dt$, and  the  last inequality holds if $\|\bh\|<(1-1/\sqrt{2}) \|\bx\|$. 

%
%
%

\section{Proof of Theorem \ref{th:mainthm}: Geometric Convergence of RWF}\label{supp:convergence}

The general structure of the proof follows that for WF in \cite{candes2015phase} and TWF in \cite{chen2015solving}. However, the proof  requires development of new bounds due to the nonsmoothness of the loss function and absolute value based measurements. On the other hand the proof is much simpler due to the lower-order loss function adopted in RWF.

The idea of the proof is to show that within the neighborhood of global optimums, RWF satisfies the {\em Regularity Condition} $\mathsf{RC}(\mu,\lambda,c)$ , i.e.,
\begin{flalign}
\left\langle \nabla \ell(\bz), \bh \right\rangle \ge \frac{\mu}{2} \left\|\nabla \ell(\bz)\right\|^2+\frac{\lambda}{2} \|\bh\|^2 \label{eq:RCsupp}
\end{flalign}
for all $\bz$ and $\bh=\bz-\bx$ obeying $\|\bh\|\le c \|\bx\|$, where $0<c<1$ is some constant. Then, as shown in \cite{chen2015solving}, once the initialization lands into this neighborhood, geometric convergence can be guaranteed, i.e., 
\begin{flalign}
	\dist^2\left(\bz+\mu\nabla\ell(\bz),\bx\right)\le (1-\mu\lambda)\dist^2(\bz,\bx),
\end{flalign}
for any $\bz$ with $\|\bz-\bx\|\le c\|\bx\|$.

To show the regularity condition, we first define a set $\cS:=\{i: 1\le i \le m, (\ba_i^T \bz)(\ba_i^T\bx)<0\}$, and then derive the following bound:
\begin{flalign}
\left\langle \nabla \ell(\bz), \bh \right\rangle &=\frac{1}{m} \sum_{i=1}^m \left(\ba_i^T\bz-|\ba_i^T\bx|\sgn(\ba_i^T\bz)\right)(\ba_i^T\bh)\nn\\
&=\frac{1}{m}\left[\sum_{i=1}^m (\ba_i^T\bh)^2+2\sum_{i\in \cS} (\ba_i^T\bx)(\ba_i^T\bh)\right] \nn\\
&\ge\frac{1}{m}\left[\sum_{i=1}^m (\ba_i^T\bh)^2-2\left|\sum_{i\in \cS} (\ba_i^T\bx)(\ba_i^T\bh)\right|\right] \nn\\
&\ge\frac{1}{m}\left[\sum_{i=1}^m (\ba_i^T\bh)^2-\sum_{i\in \cS}2\left| (\ba_i^T\bx)(\ba_i^T\bh)\right|\right]. \label{eq:condrc}
\end{flalign}
The first term in \eqref{eq:condrc} can be bounded using Lemma 3.1 in \cite{candes2013phaselift}, which we state below.
\begin{lemma} \label{lem:lemma1}
For any $0<\epsilon<1$,  if $m>c_0 n \epsilon^{-2}$, then with probability at least $1-2\exp(-c_1\epsilon^2 m)$,
\begin{flalign}
	(1-\epsilon)\|\bh\|^2\le \frac{1}{m}\sum_{i=1}^m (\ba_i^T\bh)^2\le (1+\epsilon)\|\bh\|^2
\end{flalign}
holds for all non-zero vectors $\bh\in \bbR^n$. Here, $c_0, c_1>0$ are some universal constants. 
\end{lemma}

For the second term in \eqref{eq:condrc}, we derive
\begin{flalign}
\sum_{i\in \cS} 2\left|\ba_i^T\bx\right|\left|\ba_i^T\bh\right|&\le \sum_{i\in \cS} \left[(\ba_i^T\bx)^2+(\ba_i^T\bh)^2 \right] \nn \\
&= \sum_{i=1}^m [(\ba_i^T\bx)^2+(\ba_i^T\bh)^2]\cdot \bone_{\{(\ba_i^T\bx)(\ba_i^T\bz)<0\}} \nn \\
&=  \sum_{i=1}^m [(\ba_i^T\bx)^2+(\ba_i^T\bh)^2]\cdot \bone_{\{(\ba_i^T\bx)^2+(\ba_i^T\bx)(\ba_i^T\bh)<0\}} \nn \\
&\le  \sum_{i=1}^m [(\ba_i^T\bx)^2+(\ba_i^T\bh)^2]\cdot \bone_{\{|\ba_i^T\bx|<|\ba_i^T\bh|\}} \nn\\
&\le  2\sum_{i=1}^m (\ba_i^T\bh)^2\cdot \bone_{\{|\ba_i^T\bx|<|\ba_i^T\bh|\}}.
\end{flalign}
The above equation can be further upper bounded by the following lemma.
\begin{lemma} \label{lem:lemma2}
For any $\epsilon>0$, if $m>c_0 n \epsilon^{-2}\log \epsilon^{-1}$, then with probability at least $1-C \exp(-c_1 \epsilon^2 m)$,
\begin{flalign}
	\frac{1}{m}\sum_{i=1}^m (\ba_i^T\bh)^2\cdot \bone_{\{|\ba_i^T\bx|<|\ba_i^T\bh|\}}\le \left(0.13+\epsilon\right)\|\bh\|^2 \label{eq:lemma2upper}
\end{flalign}
holds for all non-zero vectors $\bh\in \bbR^n$ satisfying $\|\bh\|\le \frac{1}{10}\|\bx\|$. Here, $c_0, c_1, C>0$ are some universal constants.
\end{lemma}
\begin{proof}
See Section \ref{sec:prooflemma2}.
\end{proof}


Therefore, combining Lemmas \ref{lem:lemma1} and \ref{lem:lemma2} with \eqref{eq:condrc}  yields 
\begin{flalign}\label{eq:rc1}
\left\langle \nabla \ell(\bz), \bh \right\rangle\ge (1-0.26-2\epsilon)\|\bh\|^2=(0.74-2\epsilon)\|\bh\|^2.
\end{flalign} 

We further provide an upper bound on $\|\nabla \ell(\bz)\|$ in the following lemma.
\begin{lemma} \label{lem:lemma3}
Fix $\delta>0$, and assume $y_i=|\ba_i^T\bx|$. Suppose that $m\ge c_0 n$ for a certain constant $c_0>0$. There exist some universal constants $c,C>0$ such that with probability  at least $1-C \exp(-c m)$,
\begin{flalign}
	\|\nabla \ell(\bz)\|\le (1+\delta)\cdot 2\|\bh\| \label{eq:lemma3uppersupp}
\end{flalign}
holds for all non-zero vectors $\bh,\bz\in \bbR^n$ satisfying $\bz=\bx+\bh$ and  $\frac{\|\bh\|}{\|\bx\|}\le \frac{1}{10}$.
\end{lemma}
\begin{proof}
See Section \ref{sec:prooflemma3}.
\end{proof}
 
Thus, applying Lemma \ref{lem:lemma3} to \eqref{eq:rc1}, we conclude that \emph{Regularity Condition} \eqref{eq:RCsupp} holds for $\mu$ and $\lambda$ satisfying
\begin{flalign}
	0.74-2\epsilon\ge \frac{\mu}{2}\cdot 4(1+\delta)^2+\frac{\lambda}{2},
\end{flalign} 
which concludes the proof. The proofs of two major lemmas are provided in the following two subsections.

\subsection{Proof of Lemma \ref{lem:lemma2}}\label{sec:prooflemma2}

We first prove bounds for any fixed $\bh\le \frac{1}{10}\|\bx\|$, and then develop a uniform bound later on.
We introduce a series of auxiliary random Lipschitz functions to approximate the indicator functions.  For $i=1,\ldots, m$, define
\begin{flalign}
	\chi_i(t):=\begin{cases}
t, &\text{if } t>(\ba_i^T\bx)^2;\\
\frac{1}{\delta}(t-(\ba_i^T\bx)^2)+(\ba_i^T\bx)^2, & \text{if }  (1-\delta)(\ba_i^T\bx)^2\le t\le (\ba_i^T\bx)^2;\\
0, & \text{else};
\end{cases}
\end{flalign}
and then $\chi_i(t)$'s are random Lipschitz functions with Lipschitz constant $\frac{1}{\delta}$. We further have 
\begin{flalign}
	|\ba_i^T\bh|^2\bone_{\{|\ba_i^T\bx|<|\ba_i^T\bh|\}}\le \chi_i(|\ba_i^T\bh|^2) \le  |\ba_i^T\bh|^2\bone_{\{(1-\delta)|\ba_i^T\bx|^2<|\ba_i^T\bh|^2\}}.
\end{flalign}

For convenience, we denote $\gamma_i:=\frac{|\ba_i^T\bh|^2}{\|\bh\|^2} \bone_{\{(1-\delta)|\ba_i^T\bx|^2<|\ba_i^T\bh|^2\}}$ and  $\theta:=\|\bh\|/\|\bx\|$. We next estimate the expectation of $\gamma_i$, by conditional expectation,
\begin{flalign}
	\mE [\gamma_i]=\int_{\Omega}\gamma_i d\bbP  =\iint_{-\infty}^\infty \mE\left[\gamma_i\big|\ba_i^T\bx=\tau_1\|\bx\|, \ba_i^T\bh=\tau_2\|\bh\|\right]\cdot f(\tau_1,\tau_2) d\tau_1 d\tau_2 ,
	\end{flalign}
where  $f(\tau_1,\tau_2)$  is the density of two joint Gaussian random variables with correlation $\rho=\frac{\bh^T\bx}{\|\bh\|\|\bx\|}\neq \pm 1$. We then continue to derive
\begin{flalign}
	\mE[\gamma_i] &= \iint_{-\infty}^\infty  \tau_2^2 \cdot\bone_{\{\sqrt{1-\delta}|\tau_1|<|\tau_2|\theta\}} \cdot f(\tau_1,\tau_2) d\tau_1 d\tau_2\nn\\
	&= \frac{1}{2\pi\sqrt{1-\rho^2}}\int_{-\infty}^\infty  \tau_2^2 \exp\left(-\frac{\tau_2^2}{2}\right) \cdot\int_{\frac{-|\tau_2|\theta}{\sqrt{1-\delta}}}^{\frac{|\tau_2|\theta}{\sqrt{1-\delta}}}\exp\left(-\frac{(\tau_1-\rho\tau_2)^2}{2(1-\rho^2)}\right) d\tau_1 d\tau_2 \label{eq:increasefun}\\
	&= \frac{1}{2\pi}\int_{-\infty}^\infty  \tau_2^2 \exp\left(-\frac{\tau_2^2}{2}\right) \cdot\int_{\frac{-\frac{|\tau_2|\theta}{\sqrt{1-\delta}}-\rho\tau_2}{\sqrt{1-\rho^2}}}^{\frac{\frac{|\tau_2|\theta}{\sqrt{1-\delta}}-\rho\tau_2}{\sqrt{1-\rho^2}}}\exp\left(-\frac{\tau^2}{2}\right) d\tau d\tau_2 \quad\quad\text{by changing variables}\nn\\
		&= \frac{1}{2\pi}\int_{-\infty}^\infty  \tau_2^2 \exp\left(-\frac{\tau_2^2}{2}\right) \cdot \sqrt{\frac{\pi}{2}}\left(\erf\left(\frac{\frac{|\tau_2|\theta}{\sqrt{1-\delta}}-\rho\tau_2}{\sqrt{1-\rho^2}}\right)-\erf\left(\frac{-\frac{|\tau_2|\theta}{\sqrt{1-\delta}}-\rho\tau_2}{\sqrt{1-\rho^2}}\right) \right) d\tau_2\nn\\
		&= \frac{1}{\sqrt{2\pi}}\int_{0}^\infty  \tau_2^2 \exp\left(-\frac{\tau_2^2}{2}\right) \cdot \left(\erf\left(\frac{(\frac{\theta}{\sqrt{1-\delta}}-\rho)\tau_2}{\sqrt{1-\rho^2}}\right)+\erf\left(\frac{(\frac{\theta}{\sqrt{1-\delta}}+\rho)\tau_2}{\sqrt{1-\rho^2}}\right) \right) d\tau_2. \label{eq:hardint}
\end{flalign}

For $|\rho|<1$, $\mE[\gamma_i]$ is a continuous function of $\rho$. For $|\rho|=1$, $\mE[\gamma_i]=0$. The last integral \eqref{eq:hardint} can be calculated numerically. Figure \ref{fig:hardintsupp} plots $\mE[\gamma_i]$ for $\theta=0.1$ and $\delta=0.01$ over $\rho\in [-1,1]$. Furthermore, \eqref{eq:increasefun} indicates that $\mE[\gamma_i]$ is monotonically increasing with both $\theta$ and $\delta$. Thus, we obtain a universal bound
\begin{flalign}
\mE[\gamma_i]\le 0.13 \quad \text{for } \theta<0.1 \text{ and } \delta=0.01,
\end{flalign}
which further implies $\mE[\chi_i(|\ba_i^T\bh|^2)]\le 0.13\|\bh\|^2$ for $ \theta<0.1$ and $\delta=0.01$.

\begin{figure}[th]
\centering 
\includegraphics[width=3.5in]{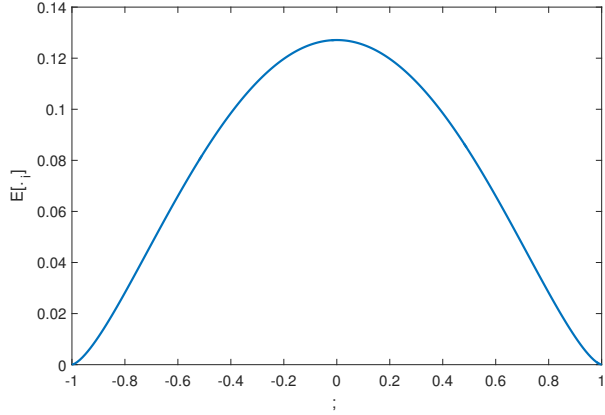}
\caption{$\mE[\gamma_i]$ with respect to $\rho$}
\label{fig:hardintsupp}
\end{figure}

Furthermore, $\chi_i(|\ba_i^T\bh|^2)$'s are sub-exponential with sub-exponential norm $\cO(\|\bh\|^2)$.
By the sub-exponential tail bound (Bernstein type) \cite{Vershynin2012}, we have
\begin{flalign}
	\cP\left[	\frac{1}{m}\sum_{i=1}^m \frac{\chi_i(|\ba_i^T\bh|^2)}{\|\bh\|^2} >\left(0.13+\epsilon\right) \right]<\exp(-cm\epsilon^2),
\end{flalign}
for some universal constant $c$, as long as $\|\bh\|\le \frac{1}{10}\|\bx\|$.

We have proved so far that the claim holds for a fixed $\bh$. 
We next obtain a uniform bound over all $\bh$ satisfying $\|\bh\|\le \frac{1}{10}\|\bx\|$. We first show the claim holds for all $\bh$ with $\|\bh\|=\frac{1}{10}\|\bx\|$ and then argue the claim holds when $\|\bh\|<\frac{1}{10}\|\bx\|$ towards the end of the proof. 
Let $\epsilon'=\epsilon\frac{\|\bx\|}{10}$ and we construct an $\epsilon'-$net $\cN_{\epsilon'}$ covering the sphere with radius $\frac{1}{10}\|\bx\|$ in $\bbR^n$ with cardinality $\left|\cN_{\epsilon'}\right|\le (1+\frac{2}{\epsilon})^n$. Then for any $\|\bh\|=\frac{1}{10}\|\bx\|$, there exists a $\bh_0\in \cN_{\epsilon'}$ such that $\|\bh-\bh_0\|\le \epsilon\|\bh\|$. Taking the union bound for all the points on the net, we claim that
\begin{flalign}
	\frac{1}{m}\sum_{i=1}^m \chi_i\left(|\ba_i^T\bh_0|^2\right) \le\left(0.13+\epsilon\right)\|\bh_0\|^2, \quad \forall \bh_0\in \cN_{\epsilon'} \label{eq:lemma2onnet}
\end{flalign}
holds with probability at least $1-(1+2/\epsilon)^n\exp(-cm\epsilon^2)$.

Since $\chi_i(t)$'s are Lipschitz functions with constant $1/\delta$, we have the following bound
\begin{flalign}
	\left|\chi_i(|\ba_i^T\bh|^2)-\chi_i(|\ba_i^T\bh_0|^2)\right|\le \frac{1}{\delta}\big||\ba_i^T\bh|^2-|\ba_i^T\bh_0|^2\big|. \label{eq:lipschitzcon}
\end{flalign}
Moreover, by \cite[Lemma 1]{chen2015solving}, we have 
\begin{flalign}
	\frac{1}{m}\|\cA(\bM)\|_1\le c_2 \|\bM\|_F, \quad \quad \text{for all symmetric rank-2 matrices } \bM\in \bbR^{n\times n}, \label{eq:eventA}
\end{flalign}
holds with probability at least $1-C\exp(-c_1m)$ as long as $m>c_0 n$ for some constants $C, c_0, c_1, c_2>0$. Consequently, on the event that \eqref{eq:eventA} holds, we have
\begin{flalign*}
	&\left|\frac{1}{m}\sum_{i=1}^m \chi_i\left(|\ba_i^T\bh|^2\right)-\frac{1}{m}\sum_{i=1}^m \chi_i\left(|\ba_i^T\bh_0|^2\right)\right|\\
	&\le \frac{1}{m}\sum_{i=1}^m \left|\chi_i\left(|\ba_i^T\bh|^2\right)-\chi_i\left(|\ba_i^T\bh_0|^2\right)\right|\\
	&\le \frac{1}{\delta}\cdot \frac{1}{m}\|\cA(\bh\bh^T-\bh_0\bh_0^T)\|_1 \quad\quad \text{because of \eqref{eq:lipschitzcon}}\\
	&\le \frac{1}{\delta} \cdot c_2\|\bh\bh^T-\bh_0\bh_0^T\|_F \quad\quad \text{because of \eqref{eq:eventA}}\\
	&\le \frac{1}{\delta}\cdot 3c_2\|\bh-\bh_0\|\cdot\|\bh\|\le 3c_3\epsilon/\delta\|\bh\|^2,
\end{flalign*}
where the last inequality is due to the Lemma 2 in \cite{chen2015solving}.

On the intersection of events that \eqref{eq:lemma2onnet} and \eqref{eq:eventA} hold, we have
\begin{flalign}
\frac{1}{m}\sum_{i=1}^m \chi_i\left(|\ba_i^T\bh|^2\right) \le\left(0.13+\epsilon+3c_3\epsilon/\delta\right)\|\bh\|^2, \label{eq:unitnet}
\end{flalign}
for all $\bh$ with $\|\bh\|=\frac{1}{10}\|\bx\|$.  

For the case when $\|\bh'\|<\frac{1}{10}\|\bx\|$,  $\bh'=\omega \bh$ for some $\bh$ satisfying $\|\bh\|=\frac{1}{10}\|\bx\|$ and $0<\omega<1$. By the definition of $\chi_i(\cdot)$, it can be verified that 
\begin{flalign}
	\chi_i(|\ba_i^T\bh'|^2)=\chi_i(|\ba_i^T(\omega\bh)|^2)\le \omega^2\chi_i(|\ba_i^T\bh|^2).
\end{flalign}
Applying \eqref{eq:unitnet}, on the same event that \eqref{eq:lemma2onnet} and \eqref{eq:eventA} hold, we have 
\begin{flalign}
\frac{1}{m}\sum_{i=1}^m \chi_i\left(|\ba_i^T\bh'|^2\right) \le\left(0.13+\epsilon+3c_3\epsilon/\delta\right)\|\bh'\|^2,
\end{flalign}
for all $\|\bh'\|<\frac{1}{10}\|\bx\|$. Since $\epsilon$ can be arbitrarily small, the proof is completed.

\subsection{Proof of Lemma \ref{lem:lemma3}}\label{sec:prooflemma3}

Denote $v_i:=\ba_i^T\bz-|\ba_i^T\bx|\sgn(\ba_i^T\bz)$. Then 
\begin{flalign}
	\nabla\ell(\bz)=\frac{1}{m}\bA^T\bv,
\end{flalign}
where $\bA$ is a matrix with each row being $\ba_i^T$ and $\bv$ is a $m-$dimensional vector with each entry being $v_i$. Thus,
\begin{flalign}
\left\|\nabla \ell(\bz)\right\|=\left\|\frac{1}{m}\bA^T\bv\right\|\le \frac{1}{m}\|\bA\|\cdot \|\bv\|\le (1+\delta)\frac{\|\bv\|}{\sqrt{m}}
\end{flalign}
as long as $m\ge c_1 n$ for some sufficiently large $c_1>0$, where the spectral norm bound $\|\bA\|\le \sqrt{m}(1+\delta)$ follows from \cite[Theorem 5.32]{Vershynin2012}.

We next bound $\|\bv\|$. Let $\bv=\bv^{(1)}+\bv^{(2)}$, where $v_i^{(1)}=\ba_i^T\bh$ and $v_i^{(2)}=2\ba_i^T\bx\bone_{\{(\ba_i^T\bz)(\ba_i^T\bx)<0\}}$. By triangle inequality, we have
$\|\bv\|\le \|\bv^{(1)}\|+\|\bv^{(2)}\|$. Furthermore, given $m>c_0n$, by \cite[Lemma 3.1]{candes2013phaselift} with probability $1-\exp(-cm)$, we have
\begin{flalign}
	\frac{1}{m}\|\bv^{(1)}\|^2=\frac{1}{m}\sum_{i=1}^{m} (\ba_i^T\bh)^2\le(1+\delta)\|\bh\|^2.
\end{flalign}
By Lemma \ref{lem:lemma2}, we have with probability $1-C\exp(-c_1 m)$
\begin{flalign}
	\frac{1}{m}\|\bv^{(2)}\|^2=\frac{1}{m}\sum_{i=1}^m 4(\ba_i^T\bx)^2\cdot \bone_{\{(\ba_i^T\bx)(\ba_i^T\bz)<0\}}\le 4(0.13+\epsilon)\|\bh\|^2.
\end{flalign}
Hence,
\begin{flalign}
\frac{\|\bv\|}{\sqrt{m}}\le [\sqrt{1+\delta}+2\sqrt{0.13+\epsilon}]\|\bh\|.
\end{flalign}
This concludes the proof. 

\section{Proof of Theorem 2: Stability to Bounded Noise}\label{sec:proofofstability}

We consider the model \eqref{eq:noisymodel} with bounded noise, i.e., $y_i=\left|\langle \ba_i,\bx \rangle\right|+w_i$ for $ i=1,\cdots,m$. The initialization analysis is similar to Appendix~\ref{supp:initialization}. To analyze the gradient loop, we consider two regimes.

$\bullet$ \textbf{Regime 1}: $c_4\|\bz\|\ge \|\bh\|\ge c_3\frac{\|\bw\|}{\sqrt{m}}$. In this regime, error contraction by each gradient step is given by
\begin{flalign}
	\dist\left(\bz+\mu\nabla \ell(\bz), \bx\right)\le (1-\rho) \dist(\bz,\bx).
\end{flalign}
It suffices to justify that $\nabla \ell(\bz)$ satisfies the $\mathsf{RC}$.  
We have
\begin{flalign}
\nabla \ell(\bz) &= \frac{1}{m}\sum_{i=1}^m \left(\ba_i^T\bz-y_i\cdot \frac{\ba_i^T\bz}{|\ba_i^T\bz|}\right) \ba_i = \underbrace{\frac{1}{m}\sum_{i=1}^m \left(\ba_i^T\bz-|\ba_i^T\bx|\cdot \frac{\ba_i^T\bz}{|\ba_i^T\bz|}\right)\ba_i}_{\nabla^{clean}\ell(\bz)} -\underbrace{\frac{1}{m}\sum_{i=1}^m \left(w_i\cdot \frac{\ba_i^T\bz}{|\ba_i^T\bz|}\right)\ba_i}_{\nabla^{noise}\ell(\bz)}.
\end{flalign}

All the proofs for Lemma \ref{lem:lemma1}, \ref{lem:lemma2} and \ref{lem:lemma3} are still valid for $\nabla^{clean} \ell(\bz)$, and thus we have
\begin{flalign}
	&\frac{1}{m}\langle \nabla^{clean} \ell(\bz), \bh\rangle \ge 0.74\|\bh\|^2,\\
	&\frac{1}{m}\left\|\nabla^{clean}\ell(\bz)\right\|\le 2(1+\delta)\|\bh\|.
\end{flalign}

Next, we analyze the contribution of the noise. Let $\tilde{w}_i=w_i\frac{\ba_i^T\bz}{|\ba_i^T\bz|}$, and then for sufficient large $m/n$, we have
\begin{flalign}
\|\nabla^{noise}\ell(\bz)\|=\left\|\frac{1}{m}\bA^T\tilde{\bw}\right\|\le \left\|\frac{1}{\sqrt{m}}\bA^T\right\|\left\|\frac{\tilde{\bw}}{\sqrt{m}}\right\|\le (1+\delta)\frac{\|\tilde{\bw}\|}{\sqrt{m}}\le (1+\delta)\frac{\|\bw\|}{\sqrt{m}},
\end{flalign}
where the second inequality is because the spectral norm bound $\|\bA\|\le \sqrt{m}(1+\delta)$ following from \cite[Theorem 5.32]{Vershynin2012}. Given the regime condition $\|\bh\|\ge c_3\frac{\|\bw\|}{\sqrt{m}}$, we further have
\begin{flalign}
&\|\nabla^{noise}\ell(\bz)\|\le  \frac{(1+\delta)}{c_3}\|\bh\|,\\
& \left|\left\langle \nabla^{noise} \ell(\bz), \bh\right\rangle\right| \le   \left\| \nabla^{noise} \ell(\bz)\right\|\cdot\|\bh\| \le \frac{(1+\delta)}{c_3}\|\bh\|^2.
\end{flalign}

Combining these together, one has
\begin{flalign}
\left\langle \nabla \ell(\bz), \bh \right\rangle & \ge \left\langle \nabla^{clean}\ell(\bz),\bh \right\rangle-\left|\left\langle  \nabla^{noise} \ell(\bz), \bh\right\rangle\right|\ge \left(0.74-\frac{(1+\delta)}{c_3}\right)\|\bh\|^2,
\end{flalign}
and
\begin{flalign}
\left\|\nabla \ell(\bz)\right\|&\le\left\|\nabla^{clean} \ell(\bz)\right\|+\left\|\nabla^{noise} \ell(\bz)\right\| \le (1+\delta)\left(2+\frac{1}{c_3}\right) \|\bh\|.
\end{flalign}

The $\mathsf{RC}$ is guaranteed if $\mu,\lambda,\epsilon$ are chosen properly, $c_3$ is sufficiently large, and $s$ is sufficiently small.

$\bullet$ \textbf{Regime 2}: Once the iterate enters the regime with $\|\bh\|\le \frac{c_3\|\bw\|}{\sqrt{m}}$, gradient update may not reduce the estimation error. However, in this regime, each move size $\mu\nabla \ell(\bz)$ is at most $\mathcal{O}(\|\bw\|/\sqrt{m})$. Then the estimation error cannot increase by more than $\|\bw\|/\sqrt{m}$ with a constant factor. Thus, one has
\begin{flalign}
	\dist\left(\bz+\mu\nabla \ell(\bz),\bx\right)\le c_5\frac{\|\bw\|}{\sqrt{m}}
\end{flalign}
for some constant $c_5$. As long as $\|\bw\|/\sqrt{m}$ is sufficiently small, it is guaranteed that $c_5\frac{\|\bw\|}{\sqrt{m}}\le c_4\|\bx\|$. If the iterate jumps out of \emph{Regime 2}, it falls into \emph{Regime 1}.

\section{Convergence of IRWF, Minibatch IRWF and Kaczmarz-PR}\label{supp:incremental}

\subsection{Proof of Theorem \ref{th:incremental}}\label{sec:proofofincremental}

Since the initialization is the same as that in Algorithm \ref{alg:rwf}, it suffices to show the convergence of gradient loops given that the initial point lands into the neighborhood of global minimums. 
To prove Theorem \ref{th:incremental}, the major step is to prove the following Proposition \ref{prop:oneiterate} which characterizes how the error of an estimate decays upon one iteration of Algorithm \ref{alg:irwf}. Once Proposition \ref{prop:oneiterate} is established, we take expectation on both sides of \eqref{eq:oneiterate} with respect to $i_{t-1}$, and apply Proposition \ref{prop:oneiterate} one more time to obtain
\begin{flalign}
\mE_{\{i_{t-1}, i_t\}}\left[\dist^2(\bz^{(t+1)},\bx)\right]\le \left(1-\frac{\rho}{n}\right)^2\dist^2(\bz^{(t-1)},\bx).
\end{flalign}
Continuing this process until the initialization point $\bz^{(0)}$ yields Theorem \ref{th:incremental}. We next focus on proving Proposition \ref{prop:oneiterate} stated bellow
\begin{proposition}\label{prop:oneiterate}
Assume the measurement vectors are independent and each $\ba_i\sim \cN(0,\bI)$.
There exist some universal constants $0<\rho,\rho_0<1$ and $c_0,c_1,c_2>0$ such that if $m\ge c_0 n$ and $\mu=\frac{\rho_0}{n}$, then with probability at least $1-c_1\exp(-c_2 m)$, we have
\begin{flalign}
	\mE_{i_t}\left[\dist^2(\bz^{(t+1)},\bx)\right]\le \left(1-\frac{\rho}{n}\right)\cdot \dist^2(\bz^{(t)},\bx)		\label{eq:oneiterate}
\end{flalign}
to hold for all $\bz^{(t)}$ satisfying $\frac{\dist(\bz^{(t)},\bx)}{\|\bz\|}\le \frac{1}{10}$.
\end{proposition}

\begin{proof}
Without loss of generality, we assume $\bz^{(t)}$ is in the neighborhood of $\bx$ (otherwise it is in the neighborhood of $-\bx$). Let $\bh=\bz^{(t)}-\bx$.

We follow the notations in Appendix \ref{supp:convergence} and let $\cS=\{i: (\ba_i^T\bx)(\ba_i^T\bz^{(t)})<0\}$. Then we have
\begin{flalign}
\mE_{i_t}&\left[\dist^2\left(\bz^{(t+1)}, \bx\right)\right] \nn\\
&=\mE_{i_t}\left[\left\|\left(\bz^{(t)}-\mu\left(\ba_{i_t}^T\bz^{(t)}-y_{i_t}\cdot \frac{\ba_{i_t}^T\bz^{(t)}}{|\ba_{i_t}^T\bz^{(t)}|}\right)\ba_{i_t}\right)-\bx\right\|^2\right] \nn\\
&=\|\bh\|^2-2\mu\mE_{i_t}\left[\ba_{i_t}^T\bh\left(\ba_{i_t}^T\bz^{(t)}-y_{i_t}\cdot\frac{\ba_{i_t}^T\bz^{(t)}}{|\ba_{i_t}^T\bz^{(t)}|}\right)\right]+\mu^2\mE_{i_t}\left[\|\ba_{i_t}\|^2\left(\ba_{i_t}^T\bz^{(t)}-y_{i_t}\cdot\frac{\ba_{i_t}^T\bz^{(t)}}{|\ba_{i_t}^T\bz^{(t)}|}\right)^2\right] \nn\\
&\overset{(a)}{=}\|\bh\|^2-\frac{2\mu}{m}\sum_{i=1}^m\left[\ba_{i}^T\bh\left(\ba_{i}^T\bz^{(t)}-y_{i}\cdot\frac{\ba_{i}^T\bz^{(t)}}{|\ba_{i}^T\bz^{(t)}|}\right)\right]+\frac{\mu^2}{m}\sum_{i=1}^m\left[\|\ba_{i}\|^2\left(\ba_{i}^T\bz^{(t)}-y_{i}\cdot\frac{\ba_{i}^T\bz^{(t)}}{|\ba_{i}^T\bz^{(t)}|}\right)^2\right] \nn\\
&=\|\bh\|^2-\frac{2\mu}{m}\left(\sum_{i=1}^m(\ba_{i}^T\bh)^2+\sum_{i\in \cS}2(\ba_{i}^T\bh)(\ba_i^T\bx)\right)+\frac{\mu^2}{m}\left(\sum_{i=1}^m\|\ba_i\|^2(\ba_i^T\bh)^2+4\sum_{i\in \cS}\|\ba_i\|^2(\ba_i^T\bx)(\ba_i^T\bz^{(t)})\right)\nn\\
&\le \|\bh\|^2-\frac{2\mu}{m}\sum_{i=1}^m(\ba_{i}^T\bh)^2+\frac{4\mu}{m}\sum_{i\in \cS}\left|(\ba_{i}^T\bh)(\ba_i^T\bx)\right|+\frac{\mu^2}{m}\sum_{i=1}^m\|\ba_i\|^2(\ba_i^T\bh)^2, \label{eq:incrementalexpansion}
\end{flalign}
where $(a)$ is due to the fact $i_t$ is sampled uniformly at random from $\{1,2,\cdots, m\}$.
Applying Lemma \ref{lem:lemma1}, we have that if $m\ge c_0 \epsilon^{-2}n$, then with probability $1-2\exp(-c_1m\epsilon^2)$
\begin{flalign*}
(1-\epsilon)\|\bh\|^2\le\frac{1}{m}\sum_{i=1}^m(\ba_{i}^T\bh)^2\le (1+\epsilon)\|\bh\|^2
\end{flalign*}
holds for all vectors $\bh$.
Furthermore, by Lemma \ref{lem:lemma2}, we have that with probability $1-C \exp(-c_1  m\epsilon^2)$,
\begin{flalign*}
\frac{1}{m}\sum_{i\in \cS}\left|(\ba_{i}^T\bh)(\ba_i^T\bx)\right|\le (0.13+\epsilon)\|\bh\|^2
\end{flalign*}
holds for all $\bh$ satisfying $\|\bh\|/\|\bx\|\le \frac{1}{10}$.

  Define an event $E_1:=\{\max_{1\le i\le m}\|\ba_i\|^2\le 6n\}$. It can be shown that $\bbP\{E_1\}\ge 1-m \exp (-1.5n)$. Then on the event $E_1$, \eqref{eq:incrementalexpansion} is further upper bounded by
  \begin{flalign}
\mE_{i_t}\left[\dist^2\left(\bz^{(t+1)}, \bx\right)\right]&\le \left(1-2\mu(1-\epsilon)+4\mu(0.13+\epsilon)+\mu^2\cdot 6n(1+\epsilon)\right)\|\bh\|^2\nn\\
	&\le \large(1-2\mu(0.74-3\epsilon-3n(1+\epsilon)\mu)\large)\|\bh\|^2.\label{eq:incrementalLast}
\end{flalign}
By choosing the step size $\mu\le \frac{0.24}{n}$, the proposition is proved.
\end{proof}

\subsection{Proof of Theorem \ref{th:blockincremental}}\label{suppsub:blockincremental}

As argued in Appendix \ref{sec:proofofincremental}, it suffices to show that one iteration of Algorithm \ref{alg:irwf} satisfies the following property.
\begin{proposition}\label{prop:oneiterateblock}
Assume the measurement vectors are independent and each $\ba_i\sim \cN(0,\bI)$.
There exist some universal constants $0<\rho,\rho_0<1$ and $c_0,c_1,c_2>0$ such that if $m\ge c_0 n$ and $\mu=\rho_0/n$ for the update rule \eqref{eq:birwfUpdate}, then with probability at least $1-c_1\exp(-c_2 m)$, we have
\begin{flalign}
	\mE_{\Gamma_t}\left[\dist^2(\bz^{(t+1)},\bx)\right]\le \left(1-\frac{k\rho}{n}\right)\cdot \dist^2(\bz^{(t)},\bx)		\label{eq:oneiterateblock}
\end{flalign}
to hold for all $\bz^{(t)}$ satisfying $\frac{\dist(\bz^{(t)},\bx)}{\|\bz\|}\le \frac{1}{10}$.
\end{proposition}

\begin{proof}
Without loss of generality, we assume $\bz^{(t)}$ is in the neighborhood of $\bx$ (otherwise it is in the neighborhood of $-\bx$). Let $\bh=\bz^{(t)}-\bx$.

We follow the notations in Appendix \ref{supp:convergence} and let $\cS=\{i: (\ba_i^T\bx)(\ba_i^T\bz^{(t)})<0\}$. Then we have
\begin{flalign}
\mE_{\Gamma_t}&\left[\dist^2\left(\bz^{(t+1)}, \bx\right)\right] \nn\\
&=\mE_{\Gamma_t}\left[\left\|\bz^{(t)}-\mu\bA_{\Gamma_t}^T\left(\bA_{\Gamma_t}\bz^{(t)}-y_{\Gamma_t}\odot \sgn(\bA_{\Gamma_t}\bz^{(t)})\right)-\bx\right\|^2\right] \nn\\
&=\|\bh\|^2-2\mu\mE_{\Gamma_t}\left[\left(\bA_{\Gamma_t}\bz^{(t)}-y_{\Gamma_t}\odot\sgn(\bA_{\Gamma_t}\bz^{(t)})\right)^T(\bA_{\Gamma_t}\bh)\right]\nn\\
&\quad+\mu^2\mE_{\Gamma_t}\left[\left(\bA_{\Gamma_t}^T\left(\bA_{\Gamma_t}\bz^{(t)}-y_{\Gamma_t}\odot\sgn(\bA_{\Gamma_t}\bz^{(t)})\right)\right)^T\left(\bA_{\Gamma_t}^T\left(\bA_{\Gamma_t}\bz^{(t)}-y_{\Gamma_t}\odot\sgn(\bA_{\Gamma_t}\bz^{(t)})\right)\right)\right] \nn\\
&\overset{(a)}{=}\|\bh\|^2-\frac{2\mu k}{m}\sum_{i=1}^m\left[\ba_{i}^T\bh\left(\ba_{i}^T\bz^{(t)}-y_{i}\cdot\frac{\ba_{i}^T\bz^{(t)}}{|\ba_{i}^T\bz^{(t)}|}\right)\right]+\frac{\mu^2k}{m}\sum_{i=1}^m\left[\|\ba_{i}\|^2\left(\ba_{i}^T\bz^{(t)}-y_{i}\cdot\frac{\ba_{i}^T\bz^{(t)}}{|\ba_{i}^T\bz^{(t)}|}\right)^2\right] \nn\\
&=\|\bh\|^2-\frac{2\mu k}{m}\left(\sum_{i=1}^m(\ba_{i}^T\bh)^2+\sum_{i\in \cS}2(\ba_{i}^T\bh)(\ba_i^T\bx)\right)+\frac{\mu^2k}{m}\left(\sum_{i=1}^m\|\ba_i\|^2(\ba_i^T\bh)^2+4\sum_{i\in \cS}\|\ba_i\|^2(\ba_i^T\bx)(\ba_i^T\bz^{(t)})\right)\nn\\
&\le \|\bh\|^2-\frac{2\mu k}{m}\sum_{i=1}^m(\ba_{i}^T\bh)^2+\frac{4\mu k}{m}\sum_{i\in \cS}\left|(\ba_{i}^T\bh)(\ba_i^T\bx)\right|+\frac{\mu^2 k}{m}\sum_{i=1}^m\|\ba_i\|^2(\ba_i^T\bh)^2, \label{eq:blockexpansion}
\end{flalign}
where $(a)$ is due to the fact that $\Gamma_t$ is uniformly chosen from all subsets of $\{1,2,\ldots, m\}$ with cardinality $k$. 

%

Following the arguments for obtaining \eqref{eq:incrementalLast}, we have
  \begin{flalign}
	\mE_{\Gamma_t}\left[\dist^2\left(\bz^{(t+1)}, \bx\right)\right]&\le \left(1-2\mu k(1-\epsilon)+4\mu k(0.13+\epsilon)+\mu^2 k\cdot 6n(1+\epsilon)\right)\|\bh\|^2\nn\\
	&\le \large(1-2\mu k(0.74-3\epsilon-3n\mu(1+\epsilon))\large)\|\bh\|^2.
\end{flalign}
By choosing the step size $\mu\le \frac{0.24}{n}$, the proposition is proved.
\end{proof}

\subsection{Proof of Theorem \ref{cor:kaczmarz}}\label{supp:subsec:kaczmarz}

Without loss of generality, we assume $\bz^{(t)}$ is in the neighborhood of $\bx$ (otherwise it is in the neighborhood of $-\bx$). Let $\bh=\bz^{(t)}-\bx$.

We follow the notations in Appendix \ref{supp:convergence} and let $\cS=\{i: (\ba_i^T\bx)(\ba_i^T\bz^{(t)})<0\}$. Then we have
\begin{flalign}
\mE_{i_t}&\left[\dist^2\left(\bz^{(t+1)}, \bx\right)\right] \nn\\
&=\mE_{i_t}\left[\left\|\left(\bz^{(t)}-\frac{1}{\|\ba_{i_t}\|^2}\left(\ba_{i_t}^T\bz^{(t)}-y_{i_t}\cdot \frac{\ba_{i_t}^T\bz^{(t)}}{|\ba_{i_t}^T\bz^{(t)}|}\right)\ba_{i_t}\right)-\bx\right\|^2\right] \nn\\
&=\|\bh\|^2-2\mE_{i_t}\left[\frac{1}{\|\ba_{i_t}\|^2}\left(\ba_{i_t}^T\bh\right)\left(\ba_{i_t}^T\bz^{(t)}-y_{i_t}\cdot\frac{\ba_{i_t}^T\bz^{(t)}}{|\ba_{i_t}^T\bz^{(t)}|}\right)\right]+\mE_{i_t}\left[\frac{1}{\|\ba_{i_t}\|^2}\left(\ba_{i_t}^T\bz^{(t)}-y_{i_t}\cdot\frac{\ba_{i_t}^T\bz^{(t)}}{|\ba_{i_t}^T\bz^{(t)}|}\right)^2\right] \nn\\
&\overset{(a)}{=}\|\bh\|^2-\frac{2}{m}\sum_{i=1}^m\left[\frac{1}{\|\ba_{i}\|^2}\left(\ba_{i}^T\bh\right)\left(\ba_{i}^T\bz^{(t)}-y_{i}\cdot\frac{\ba_{i}^T\bz^{(t)}}{|\ba_{i}^T\bz^{(t)}|}\right)\right]+\frac{1}{m}\sum_{i=1}^m\left[\frac{1}{\|\ba_{i}\|^2}\left(\ba_{i}^T\bz^{(t)}-y_{i}\cdot\frac{\ba_{i}^T\bz^{(t)}}{|\ba_{i}^T\bz^{(t)}|}\right)^2\right] \nn\\
&=\|\bh\|^2-\frac{2}{m}\left(\sum_{i=1}^m\frac{(\ba_{i}^T\bh)^2}{\|\ba_{i}\|^2}+\sum_{i\in \cS}\frac{2(\ba_{i}^T\bh)(\ba_i^T\bx)}{\|\ba_{i}\|^2}\right)+\frac{1}{m}\left(\sum_{i=1}^m\frac{(\ba_{i}^T\bh)^2}{\|\ba_{i}\|^2}+4\sum_{i\in \cS}\frac{(\ba_i^T\bx)(\ba_i^T\bz^{(t)})}{\|\ba_{i}\|^2}\right)\nn\\
&= \|\bh\|^2-\frac{1}{m}\sum_{i=1}^m\frac{(\ba_{i}^T\bh)^2}{\|\ba_{i}\|^2}+\frac{4}{m}\sum_{i\in \cS}\frac{(\ba_{i}^T\bx)^2}{\|\ba_{i}\|^2} \label{eq:kaczmarzexpansion}
\end{flalign}
where $(a)$ is due to the fact that $i_t$ is sampled uniformly at random from $\{1,2,\cdots, m\}$.
By the spectral case of Lemma 5.20 in \cite{Vershynin2012}, $\{\sqrt{n}\frac{\ba_i}{\|\ba_i\|}\}_{i=1}^m$ are independent isotropic random vectors in $\bbR^n$ and hence
\begin{flalign*}
\mE\left[n\frac{(\ba_i^T\bh)^2}{\|\ba_i\|^2}\right]=\|\bh\|^2.
\end{flalign*}  
Moreover, $\{\sqrt{n}\frac{\ba_i}{\|\ba_i\|}\}_{i=1}^m$ are sub-Gaussian and the sub-Gaussian norm is bounded by an absolute constant. Thus, we have that if $m\ge c_0 \epsilon^{-2}n$, then with probability $1-2\exp(-c_1m\epsilon^2)$,
\begin{flalign*}
\frac{1}{m}\sum_{i=1}^m\frac{(\ba_{i}^T\bh)^2}{\|\ba_{i}\|^2}\ge \frac{(1-\epsilon)}{n}\|\bh\|^2.
\end{flalign*}
holds for all vectors $\bh$.
By Lemma \ref{lem:lemma2}, we have that with probability $1-C \exp(-c_1  m\epsilon^2)$
\begin{flalign*}
\frac{1}{m}\sum_{i\in \cS}\left|\ba_i^T\bx\right|^2\le \frac{1}{m}\sum_{i=1}^m\left|\ba_i^T\bh\right|^2\bone_{\{|\ba_i^T\bx|<|\ba_i^T\bh|\}}\le (0.13+\epsilon)\|\bh\|^2
\end{flalign*}
holds for all $\bh$ satisfying $\|\bh\|/\|\bx\|\le \frac{1}{10}$.

  Define an event $E_2:=\{\min_{1\le i\le m}\|\ba_i\|^2\ge \frac{2}{3} n\}$. It can be shown that $\bbP\{E_2\}\ge 1-m \exp (-n/12)$. Then on the event $E_2$, \eqref{eq:kaczmarzexpansion} is further upper bounded by
  \begin{flalign}
\mE_{i_t}\left[\dist^2\left(\bz^{(t+1)}, \bx\right)\right]\le \left(1-\frac{1-\epsilon}{n}+\frac{6(0.13+\epsilon)}{n}\right)\|\bh\|^2\le \left(1-\frac{0.22-7\epsilon}{n}\right)\|\bh\|^2,
\end{flalign}
which concludes the proof.


\begin{thebibliography}{10}
	
	\bibitem{candes2015phase}
	E.~J. Cand\`es, X.~Li, and M.~Soltanolkotabi.
	\newblock Phase retrieval via wirtinger flow: Theory and algorithms.
	\newblock {\em IEEE Transactions on Information Theory}, 61(4):1985--2007,
	2015.
	
	\bibitem{chen2015solving}
	Y.~Chen and E.~Candes.
	\newblock Solving random quadratic systems of equations is nearly as easy as
	solving linear systems.
	\newblock In {\em Advances in Neural Information Processing Systems (NIPS)}.
	2015.
	
	\bibitem{drenth2007X}
	J.~Drenth.
	\newblock {\em X-Ray Crystallography}.
	\newblock Wiley Online Library, 2007.
	
	\bibitem{miao1999extending}
	J.~Miao, P.~Charalambous, J.~Kirz, and D.~Sayre.
	\newblock Extending the methodology of x-ray crystallography to allow imaging
	of micrometre-sized non-crystalline specimens.
	\newblock {\em Nature}, 400(6742):342--344, 1999.
	
	\bibitem{miao2008extending}
	J.~Miao, T.~Ishikawa, Q.~Shen, and T.~Earnest.
	\newblock Extending x-ray crystallography to allow the imaging of
	noncrystalline materials, cells, and single protein complexes.
	\newblock {\em Annu. Rev. Phys. Chem.}, 59:387--410, 2008.
	
	\bibitem{gerchberg1972practical}
	R.~W. Gerchberg.
	\newblock A practical algorithm for the determination of phase from image and
	diffraction plane pictures.
	\newblock {\em Optik}, 35:237, 1972.
	
	\bibitem{fienup1982phase}
	J.~R. Fienup.
	\newblock Phase retrieval algorithms: a comparison.
	\newblock {\em Applied Optics}, 21(15):2758--2769, 1982.
	
	\bibitem{chai2011array}
	A.~Chai, M.~Moscoso, and G.~Papanicolaou.
	\newblock Array imaging using intensity-only measurements.
	\newblock {\em Inverse Problems}, 27(1), 2011.
	
	\bibitem{candes2013phaselift}
	E.~J. Cand\`es, T.~Strohmer, and V.~Voroninski.
	\newblock Phaselift: Exact and stable signal recovery from magnitude
	measurements via convex programming.
	\newblock {\em Communications on Pure and Applied Mathematics},
	66(8):1241--1274, 2013.
	
	\bibitem{gross2015improved}
	D.~Gross, F.~Krahmer, and R.~Kueng.
	\newblock Improved recovery guarantees for phase retrieval from coded
	diffraction patterns.
	\newblock {\em Applied and Computational Harmonic Analysis}, 2015.
	
	\bibitem{waldspurger2015phase}
	I.~Waldspurger, A.~d{'}Aspremont, and S.~Mallat.
	\newblock Phase recovery, maxcut and complex semidefinite programming.
	\newblock {\em Mathematical Programming}, 149(1-2):47--81, 2015.
	
	\bibitem{shechtman2015phase}
	Y.~Shechtman, Y.~C. Eldar, O.~Cohen, H.~N. Chapman, J.~Miao, and M.~Segev.
	\newblock Phase retrieval with application to optical imaging: a contemporary
	overview.
	\newblock {\em IEEE Signal Processing Magazine}, 32(3):87--109, 2015.
	
	\bibitem{sujay2013phase}
	P.~Netrapalli, P.~Jain, and S.~Sanghavi.
	\newblock Phase retrieval using alternating minimization.
	\newblock {\em Advances in Neural Information Processing Systems (NIPS)}, 2013.
	
	\bibitem{wei2015solving}
	K.~Wei.
	\newblock Solving systems of phaseless equations via kaczmarz methods: a proof
	of concept study.
	\newblock {\em Inverse Problems}, 31(12):125008, 2015.
	
	\bibitem{li2015phase}
	G.~Li, Y.~Gu, and Y.~M. Lu.
	\newblock Phase retrieval using iterative projections: Dynamics in the large
	systems limit.
	\newblock In {\em The 53rd Annual Allerton Conference on Communication,
		Control, and Computing}, 2015.
	
	\bibitem{kolte2016phase}
	R.~Kolte and A.~{\"O}zg{\"u}r.
	\newblock Phase retrieval via incremental truncated wirtinger flow.
	\newblock {\em arXiv preprint arXiv:1606.03196}, 2016.
	
	\bibitem{zhang2016provable}
	H.~Zhang, Y.~Chi, and Y.~Liang.
	\newblock Provable non-convex phase retrieval with outliers: Median truncated
	wirtinger flow.
	\newblock {\em arXiv preprint arXiv:1603.03805}, 2016.
	
	\bibitem{cai2015optimal}
	T.~T. Cai, X.~Li, and Z.~Ma.
	\newblock Optimal rates of convergence for noisy sparse phase retrieval via
	thresholded wirtinger flow.
	\newblock {\em arXiv preprint arXiv:1506.03382}, 2015.
	
	\bibitem{sun2016geometric}
	J.~Sun, Q.~Qu, and J.~Wright.
	\newblock A geometric analysis of phase retrieval.
	\newblock {\em arXiv preprint arXiv:1602.06664}, 2016.
	
	\bibitem{sanghavi2016local}
	S.~Sanghavi, R.~Ward, and C.~D. White.
	\newblock The local convexity of solving systems of quadratic equations.
	\newblock {\em Results in Mathematics}, pages 1--40, 2016.
	
	\bibitem{wang2016solving}
	G.~Wang, G.~B. Giannakis, and Y.~C. Eldar.
	\newblock Solving systems of random quadratic equations via truncated amplitude
	flow.
	\newblock {\em arXiv preprint arXiv:1605.08285}, 2016.
	
	\bibitem{keshavan2010matrix}
	R.~H. Keshavan, A.~Montanari, and S.~Oh.
	\newblock Matrix completion from a few entries.
	\newblock {\em IEEE Transactions on Information Theory}, 56(6):2980 --2998,
	June 2010.
	
	\bibitem{jain2013low}
	P.~Jain, P.~Netrapalli, and S.~Sanghavi.
	\newblock Low-rank matrix completion using alternating minimization.
	\newblock In {\em Proceedings of the forty-fifth annual ACM symposium on Theory
		of computing}, 2013.
	
	\bibitem{sun2014guaranteed}
	R.~Sun and Z.-Q. Luo.
	\newblock Guaranteed matrix completion via non-convex factorization.
	\newblock {\em arXiv preprint arXiv:1411.8003}, 2014.
	
	\bibitem{hardt2014understanding}
	M.~Hardt.
	\newblock Understanding alternating minimization for matrix completion.
	\newblock In {\em Foundations of Computer Science (FOCS), 2014 IEEE 55th Annual
		Symposium on}, pages 651--660. IEEE, 2014.
	
	\bibitem{de2015global}
	C.~De~Sa, K.~Olukotun, and C.~R{\'e}.
	\newblock Global convergence of stochastic gradient descent for some non-convex
	matrix problems.
	\newblock {\em arXiv preprint arXiv:1411.1134v3}, 2015.
	
	\bibitem{zheng2016convergence}
	Q.~Zheng and J.~Lafferty.
	\newblock Convergence analysis for rectangular matrix completion using
	burer-monteiro factorization and gradient descent.
	\newblock {\em arXiv preprint arXiv:1605.07051}, 2016.
	
	\bibitem{jin2016provable}
	C.~Jin, S.~M. Kakade, and P.~Netrapalli.
	\newblock Provable efficient online matrix completion via non-convex stochastic
	gradient descent.
	\newblock {\em arXiv preprint arXiv:1605.08370}, 2016.
	
	\bibitem{ge2016matrix}
	R.~Ge, J.~D. Lee, and T.~Ma.
	\newblock Matrix completion has no spurious local minimum.
	\newblock {\em arXiv preprint arXiv:1605.07272}, 2016.
	
	\bibitem{bhojanapalli2016global}
	S.~Bhojanapalli, B.~Neyshabur, and N.~Srebro.
	\newblock Global optimality of local search for low rank matrix recovery.
	\newblock {\em arXiv preprint arXiv:1605.07221}, 2016.
	
	\bibitem{chen2015fast}
	Y.~Chen and M.~J. Wainwright.
	\newblock Fast low-rank estimation by projected gradient descent: General
	statistical and algorithmic guarantees.
	\newblock {\em arXiv preprint arXiv:1509.03025}, 2015.
	
	\bibitem{tu2015low}
	S.~Tu, R.~Boczar, M.~Soltanolkotabi, and B.~Recht.
	\newblock Low-rank solutions of linear matrix equations via procrustes flow.
	\newblock {\em arXiv preprint arXiv:1507.03566}, 2015.
	
	\bibitem{zheng2015convergent}
	Q.~Zheng and J.~Lafferty.
	\newblock A convergent gradient descent algorithm for rank minimization and
	semidefinite programming from random linear measurements.
	\newblock In {\em Advances in Neural Information Processing Systems (NIPS)},
	2015.
	
	\bibitem{park2016provable}
	D.~Park, A.~Kyrillidis, S.~Bhojanapalli, C.~Caramanis, and S.~Sanghavi.
	\newblock Provable non-convex projected gradient descent for a class of
	constrained matrix optimization problems.
	\newblock {\em arXiv preprint arXiv:1606.01316}, 2016.
	
	\bibitem{wei2015guarantees}
	K.~Wei, J.-F. Cai, T.~F. Chan, and S.~Leung.
	\newblock Guarantees of riemannian optimization for low rank matrix recovery.
	\newblock {\em arXiv preprint arXiv:1511.01562}, 2015.
	
	\bibitem{netrapalli2014non}
	P.~Netrapalli, U.~Niranjan, S.~Sanghavi, A.~Anandkumar, and P.~Jain.
	\newblock Non-convex robust pca.
	\newblock In {\em Advances in Neural Information Processing Systems (NIPS)},
	2014.
	
	\bibitem{anandkumar2015tensor}
	A.~Anandkumar, P.~Jain, Y.~Shi, and U.~Niranjan.
	\newblock Tensor vs matrix methods: Robust tensor decomposition under block
	sparse perturbations.
	\newblock {\em arXiv preprint arXiv:1510.04747}, 2015.
	
	\bibitem{arora2015simple}
	S.~Arora, R.~Ge, T.~Ma, and A.~Moitra.
	\newblock Simple, efficient, and neural algorithms for sparse coding.
	\newblock {\em arXiv preprint arXiv:1503.00778}, 2015.
	
	\bibitem{sun2015complete}
	J.~Sun, Q.~Qu, and J.~Wright.
	\newblock Complete dictionary recovery using nonconvex optimization.
	\newblock In {\em Proceedings of the 32nd International Conference on Machine
		Learning (ICML)}, 2015.
	
	\bibitem{bandeira2016low}
	A.~S. Bandeira, N.~Boumal, and V.~Voroninski.
	\newblock On the low-rank approach for semidefinite programs arising in
	synchronization and community detection.
	\newblock In {\em 29th Annual Conference on Learning Theory}, 2016.
	
	\bibitem{boumal2016nonconvex}
	N.~Boumal.
	\newblock Nonconvex phase synchronization.
	\newblock {\em arXiv preprint arXiv:1601.06114}, 2016.
	
	\bibitem{lee2015blind}
	K.~Lee, Y.~Li, M.~Junge, and Y.~Bresler.
	\newblock Blind recovery of sparse signals from subsampled convolution.
	\newblock {\em arXiv preprint arXiv:1511.06149}, 2015.
	
	\bibitem{li2016rapid}
	X.~Li, S.~Ling, T.~Strohmer, and K.~Wei.
	\newblock Rapid, robust, and reliable blind deconvolution via nonconvex
	optimization.
	\newblock {\em arXiv preprint arXiv:1606.04933}, 2016.
	
	\bibitem{burke2005robust}
	J.~V. Burke, A.~S. Lewis, and M.~L. Overton.
	\newblock A robust gradient sampling algorithm for nonsmooth, nonconvex
	optimization.
	\newblock {\em SIAM Journal on Optimization}, 15(3):751--779, 2005.
	
	\bibitem{kiwiel2007convergence}
	K.~C. Kiwiel.
	\newblock Convergence of the gradient sampling algorithm for nonsmooth
	nonconvex optimization.
	\newblock {\em SIAM Journal on Optimization}, 18(2):379--388, 2007.
	
	\bibitem{ochs2015iteratively}
	P.~Ochs, A.~Dosovitskiy, T.~Brox, and T.~Pock.
	\newblock On iteratively reweighted algorithms for nonsmooth nonconvex
	optimization in computer vision.
	\newblock {\em SIAM Journal on Imaging Sciences}, 8(1):331--372, 2015.
	
	\bibitem{krizhevsky2012imagenet}
	A.~Krizhevsky, I.~Sutskever, and G.~E. Hinton.
	\newblock Imagenet classification with deep convolutional neural networks.
	\newblock In {\em Advances in neural information processing systems (NIPS)},
	2012.
	
	\bibitem{glorot2011deep}
	X.~Glorot, A.~Bordes, and Y.~Bengio.
	\newblock Deep sparse rectifier neural networks.
	\newblock In {\em International Conference on Artificial Intelligence and
		Statistics (AISTATS)}, 2011.
	
	\bibitem{cai2013distributions}
	T.~T. Cai, J.~Fan, and T.~Jiang.
	\newblock Distributions of angles in random packing on spheres.
	\newblock {\em Journal of Machine Learning Research}, 14(1):1837--1864, 2013.
	
	\bibitem{kruger2003frechet}
	A.~Y. Kruger.
	\newblock On fr{\'e}chet subdifferentials.
	\newblock {\em Journal of Mathematical Sciences}, 116(3):3325--3358, 2003.
	
	\bibitem{chi2016kaczmarz}
	Y.~Chi and Y.~M. Lu.
	\newblock Kaczmarz method for solving quadratic equations.
	\newblock {\em IEEE Signal Processing Letters}, 23(9):1183--1187, 2016.
	
	\bibitem{moulines2011non}
	E.~Moulines and F.~R. Bach.
	\newblock Non-asymptotic analysis of stochastic approximation algorithms for
	machine learning.
	\newblock In {\em Advances in Neural Information Processing Systems (NIPS)},
	2011.
	
	\bibitem{needell2016stochastic}
	D.~Needell, N.~Srebro, and R.~Ward.
	\newblock Stochastic gradient descent, weighted sampling, and the randomized
	kaczmarz algorithm.
	\newblock {\em Mathematical Programming}, 155(1-2):549--573, 2016.
	
	\bibitem{kaczmarz1937angenaherte}
	S.~Kaczmarz.
	\newblock Angen{\"a}herte aufl{\"o}sung von systemen linearer gleichungen.
	\newblock {\em Bulletin International de l’Academie Polonaise des Sciences et
		des Lettres}, 35:355--357, 1937.
	
	\bibitem{strohmer2009randomized}
	T.~Strohmer and R.~Vershynin.
	\newblock A randomized kaczmarz algorithm with exponential convergence.
	\newblock {\em Journal of Fourier Analysis and Applications}, 15(2):262--278,
	2009.
	
	\bibitem{zouzias2013randomized}
	A.~Zouzias and N.~M. Freris.
	\newblock Randomized extended kaczmarz for solving least squares.
	\newblock {\em SIAM Journal on Matrix Analysis and Applications},
	34(2):773--793, 2013.
	
	\bibitem{xu2015minimal}
	Z.~Xu.
	\newblock The minimal measurement number for low-rank matrices recovery.
	\newblock {\em arXiv preprint arXiv:1505.07204}, 2015.
	
	\bibitem{fogel2013phase}
	F.~Fogel, I.~Waldspurger, and A.~d'Aspremont.
	\newblock Phase retrieval for imaging problems.
	\newblock {\em arXiv preprint arXiv:1304.7735}, 2013.
	
	\bibitem{lee2016gradient}
	J.~D. Lee, M.~Simchowitz, M.~I. Jordan, and B.~Recht.
	\newblock Gradient descent converges to minimizers.
	\newblock {\em arXiv preprint arXiv:1602.04915}, 2016.
	
	\bibitem{ge2015escaping}
	R.~Ge, F.~Huang, C.~Jin, and Y.~Yuan.
	\newblock Escaping from saddle points---online stochastic gradient for tensor
	decomposition.
	\newblock {\em arXiv preprint arXiv:1503.02101}, 2015.
	
	\bibitem{Vershynin2012}
	R.~Vershynin.
	\newblock Introduction to the non-asymptotic analysis of random matrices.
	\newblock {\em Compressed Sensing, Theory and Applications}, pages 210 -- 268,
	2012.
	
	\bibitem{donahue1964products}
	J.~D. Donahue.
	\newblock Products and quotients of random variables and their applications.
	\newblock Technical report, DTIC Document, 1964.
	
\end{thebibliography}
\end{document}